\documentclass[sigconf]{acmart}

\AtBeginDocument{%
  \providecommand\BibTeX{{%
    \normalfont B\kern-0.5em{\scshape i\kern-0.25em b}\kern-0.8em\TeX}}}

\setcopyright{acmcopyright}
\copyrightyear{2023}
\acmYear{2023}
\setcopyright{acmlicensed}\acmConference[KDD '23]{Proceedings of the 29th ACM SIGKDD Conference on Knowledge Discovery and Data Mining}{August 6--10, 2023}{Long Beach, CA, USA}
\acmBooktitle{Proceedings of the 29th ACM SIGKDD Conference on Knowledge Discovery and Data Mining (KDD '23), August 6--10, 2023, Long Beach, CA, USA}
\acmPrice{15.00}
\acmDOI{10.1145/3580305.3599532}
\acmISBN{979-8-4007-0103-0/23/08}
\usepackage{xspace}
\usepackage{subfigure}
\usepackage{amsmath}
\usepackage{booktabs}
\usepackage{multirow}
\usepackage{graphicx}
\usepackage{amsthm}
\definecolor{LightGray}{gray}{0.9}
\newtheorem{theorem}{Theorem}[section]
\newtheorem{definition}{Definition}[section]

\usepackage{xcolor}
\begin{document}

\title{$\kappa$HGCN:Tree-likeness Modeling via Continuous and Discrete Curvature Learning}


\author{Menglin Yang}
\authornotemark
\affiliation{%
  \institution{The Chinese University of Hong Kong}
  \city{Hong Kong SAR}
  \country{China}
}

\author{Min Zhou}
\affiliation{
  \institution{Huawei Technologies Co., Ltd.}
  \city{Shenzhen}
  \country{China}
}
\authornote{Corresponding authors: \\ ~~ Menglin Yang (mlyang.cuhk@outlook.com); \\ ~~ Min Zhou (zhoumin27@huawei.com).}

\author{Lujia Pan}
\affiliation{%
  \institution{Huawei Technologies Co., Ltd.}
  \city{Shenzhen}
  \country{China}
}

\author{Irwin King}
\affiliation{%
  \institution{The Chinese University of Hong Kong}
  \city{Hong Kong SAR}
  \country{China}
}

\renewcommand{\shortauthors}{Menglin Yang, Min Zhou, Lujia Pan, \& Irwin King}

\newcommand{\method}{{$\kappa$}HGCN\xspace}
\newcommand{\aggmethod}{HMP\xspace}
\newcommand{\hr}{$\kappa$HC\xspace}
\newcommand{\fig}{\figureautorefname}
\newcommand{\tab}{TABLE\xspace}
\begin{abstract}
The prevalence of tree-like structures, encompassing hierarchical structures and power law distributions, exists extensively in real-world applications, including recommendation systems, ecosystems, financial networks, social networks, etc. Recently, the exploitation of hyperbolic space for tree-likeness modeling has garnered considerable attention owing to its exponential growth volume. Compared to the flat Euclidean space, the curved hyperbolic space provides a more amenable and embeddable room, especially for datasets exhibiting implicit tree-like architectures. 
However, the intricate nature of real-world tree-like data presents a considerable challenge, as it frequently displays a \textit{heterogeneous} composition of tree-like, flat, and circular regions. The direct embedding of such heterogeneous structures into a homogeneous embedding space (i.e., hyperbolic space) inevitably leads to heavy distortions.
To mitigate the aforementioned shortage, this study endeavors to explore the curvature between discrete structure and continuous learning space, aiming at encoding the message conveyed by the network topology in the learning process, thereby improving tree-likeness modeling. 
To the end, a curvature-aware hyperbolic graph convolutional neural network, \method, is proposed, which utilizes the curvature to guide message passing and improve long-range propagation.
Extensive experiments on node classification and link prediction tasks verify the superiority of the proposal as it consistently outperforms various competitive models by a large margin.
\end{abstract}




\maketitle

\section{Introduction}
	\begin{figure}[!tp]
    \centering
    \includegraphics[width=0.42\textwidth]{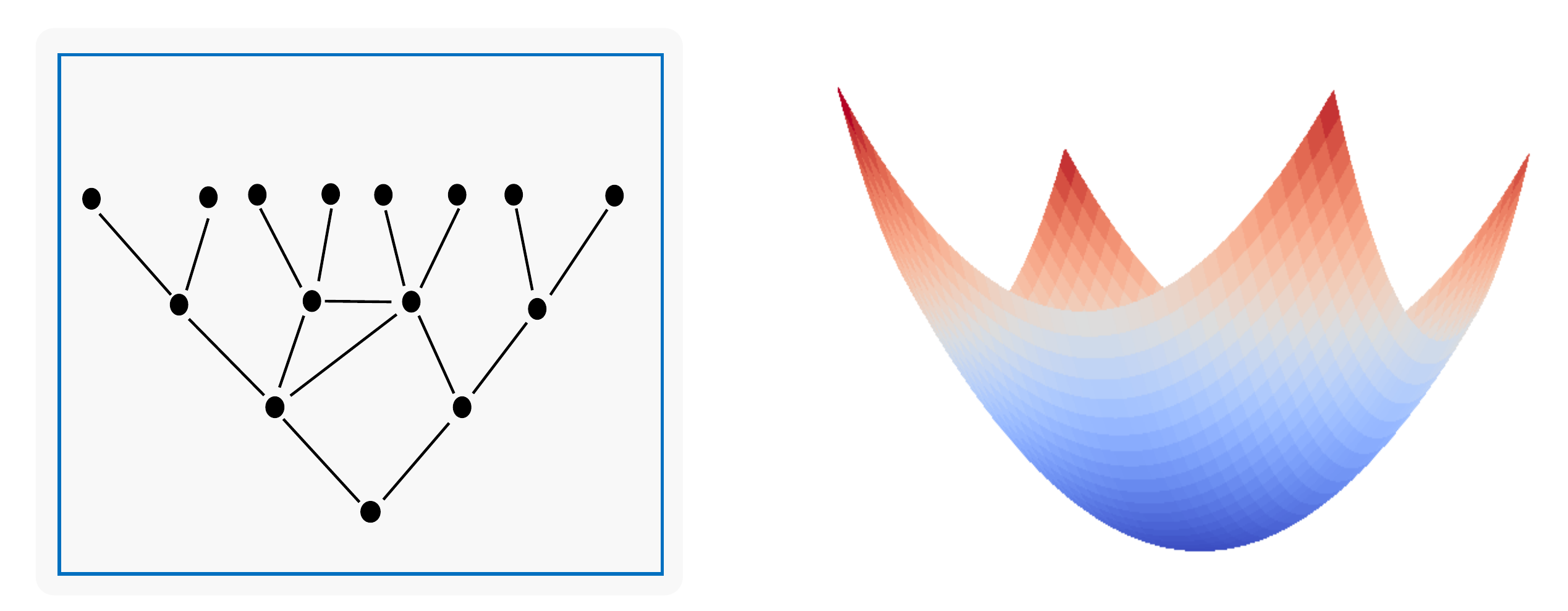}
    \caption{The (tree-like) graph $G$ in the left subfigure can be considered as a discrete approximation to the hyperbolic manifold $\mathcal{M}$ in the right subfigure; on the other hand, the hyperbolic manifold $\mathcal{M}$ can also be approximated as the graph $G$.}
    \label{fig:tree-likeness}
    \vspace{-15pt}
    \end{figure}
Tree-like structures refer to networks, systems, or data organizations that resemble a tree in their architecture, with nodes branching out from the root into multiple levels. They are widely  observed in various real-world domains~\cite{adcock2013tree,abu2016metric,narayan2011large,shavitt2008hyperbolic,zhou2022telegraph,liu2022discovering}, such as recommendation systems, financial networks, and social networks. 

Recently, the utilization of hyperbolic space for modeling tree-like structures has garnered substantial attention ~\cite{nickel2017poincare,nickel2018learning,hgcn2019,liu2019HGNN,zhang2021hyperbolic,HNN,HNN++,yang2022hyperbolic,peng2021hyperbolic}. Compared with the zero curvature Euclidean space, one key property of negative curvature hyperbolic space is that it expand exponentially, making it can be considered as a continuous tree and vice versa (as shown in Figure~\ref{fig:tree-likeness}).  
In other words, hyperbolic space allows for the efficient representation of a tree-like structure, as it enables nodes to be spread apart as they move away from the root, preventing crowding and overlapping of nodes as is commonly observed in Euclidean space. Additionally, hyperbolic space allows for exponential growth in the number of nodes in a given area, which is well-suited for modeling the exponential growth of trees. 

However, the real-world dataset often deviates from a pure tree configuration  and manifests in a labyrinthine complexity, posing large challenges for tree-likeness modeling within hyperbolic space
~\cite{zhu2020gil,wang2021mixed}.
For instance, biological taxonomies, which depict the hierarchical relationships between different species from a \textit{global} view, often exhibit both extensive \textit{local} flat regions where multiple species are closely related and \textit{local} circular regions where species connections are less defined.
The local structure of a tree-like graph exhibits a heterogeneous blend of tree-like, flat, and circular patterns, resulting in difficulties in uniformly embedding the data into a homogeneous embedding space and thereby engendering structural biases and distortions.

\begin{figure}[!tp]
\centering
\includegraphics[width=0.47\textwidth]{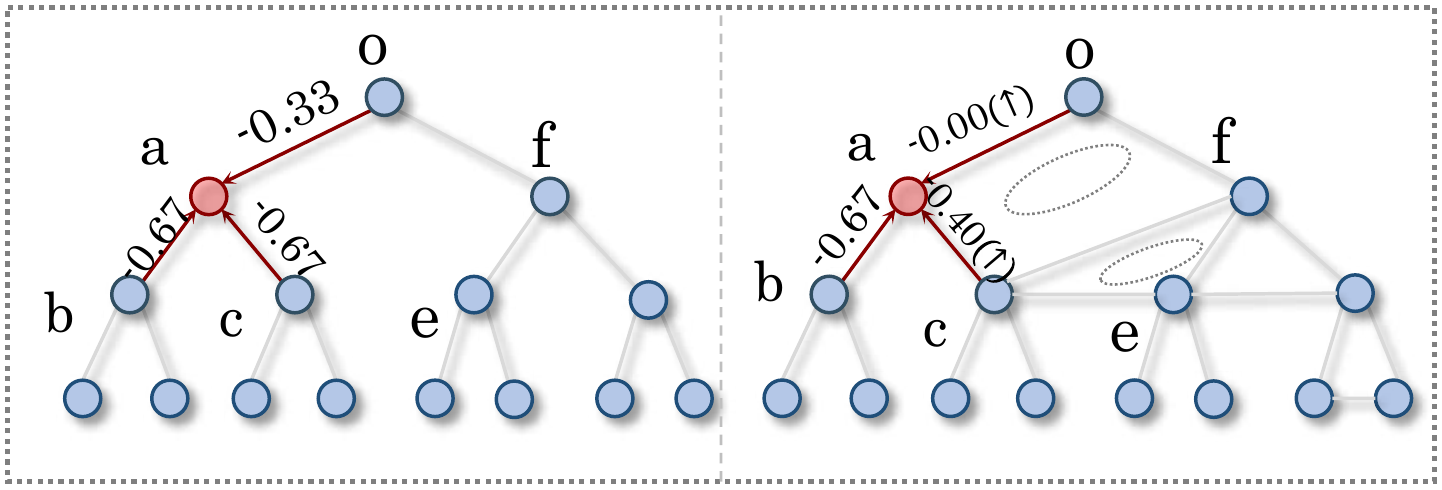}
\caption{Illustration of two tree-like graphs: $T_1$ (left) is pure and $T_2$ (right) has a local heterogeneous structure. The arrows indicate the aggregation flow and the values on the edges represent the edge curvature, denoted as $\kappa^1$ for $T_1$ and $\kappa^2$ for $T_2$. By incorporating curvature $\kappa$ into the neighboring aggregation, the detection of local structures is improved. This is achieved by assigning asymmetric weights ($\kappa_{a,o}^1$ vs $\kappa_{a,b}^1$) to nodes at different levels and results in larger values for nodes in a circular shape ($\kappa_{a,o}^1$ vs $\kappa_{a,o}^2$ and $\kappa_{a,c}^1$ vs $\kappa_{a,c}^2$).
}
\label{fig:manfiold_curvature}
\end{figure}
    
To mitigate the limitations, this study seeks to examine the intersection between the discrete structure of the data and the continuous learning space. 
The aim is to encode the information inherent in the network topology in a manner that is both effective and imbued with structural inductive bias, thereby enhancing the performance of downstream tasks.
From a geometric perspective, the quality of the embedding in geometric learning depends on the compatibility between the intrinsic graph structure and the embedding space~\cite{gu2019learning}. In light of this principle, we employs the concept of \emph{curvature} to guide tree-likeness modeling in hyperbolic space.
As shown in Figure~\ref{fig:manfiold_curvature}, the incorporation of curvature information offers a more comprehensive grasp of the local shape characteristics, facilitating the representation of the shape and contours of diverse regions within the learning space. 
    
In Riemannian geometry, curvature measures the deviation of a geometric object from being flat, originally defined on continuous smooth manifolds~\cite{lee2018introduction}. Smooth manifolds with positive, zero, or negative curvature are spherical, Euclidean, and hyperbolic spaces, respectively. This concept has been extended to discrete objects such as graphs, where curvature describes the deviation of a local pair of neighborhoods from a "flat" case.

Graph curvature, analogous to curvature in the realm of continuous geometry, consists of Gaussian curvature, Ricci curvature, and mean curvature. These components have unique roles: Gaussian curvature quantifies the local curvature at vertices, Ricci curvature allocates curvature to the edges, and mean curvature provides an overall metric for the entire graph. In this work, we focus on Ricci curvature for graph convolution and edge-based filtering. Several definitions have been proposed about Ricci curvature, including Ollivier Ricci curvature~\cite{ollivier2009ricci}, Forman Ricci curvature~\cite{forman2003bochner}, Balanced Forman curvature~\cite{topping2021understanding}.
Ricci curvature controls the overlap between two distance balls by considering the radii of the balls and the distance between their centers~\cite{jost2014Ricci_triangles}. Furthermore, the lower bound of Ricci curvature can reveal valuable global geometric and topological information~\cite{Myers_theorem}. It is also an effective indicator of tree-like, flat, and cyclic areas, making it well-suited for integration into hyperbolic space to capture asymmetries, biases, and hierarchies.

Overall, we put forward a novel framework: curvature-aware hyperbolic graph convolutional neural network (\method) for effectively modeling tree-like datasets with complex structures. Specifically, \method leverages the discrete Ricci curvature to guide message passing and dynamically adapts the global continuous hyperbolic curvature. 
Through empirical evaluations on diverse datasets and tasks, we confirm the superiority of the \method, as it consistently outperforms existing baselines by substantial margins. The major contributions of this work are summarized as follows:

\begin{itemize}
    \item We design a novel hyperbolic geometric learning framework that encapsulates the graph Ricci curvature into the continuous embedding space, producing less distortion, powerful expressiveness, and topology-aware embeddings;
    \item We present a new message technique for hyperbolic graph embedding, and we further prove that it produces a smaller (larger) embedding distance when larger (smaller) curvature is involved, which well handles the inconsistency between the local structure and global curvature of embedding space;
    \item Extensive experiments demonstrate that the proposed model \method achieves significant improvements over various baselines on link prediction and node classification tasks.
\end{itemize}

    \label{context:background}

\section{Related Work}
For tree-likeness modeling, we mainly review the latest research techniques including graph neural networks and hyperbolic geometry. In addition to this, we also review the recent advancements in curvature and curvature-based learning.

\subsection{Graph Neural Networks}
Tree-structured data can often be represented as graphs. In recent years, graph neural networks (GNNs) have gained significant attention within the graph learning community. The main concept behind GNNs is a message-passing mechanism that aggregates information from neighboring nodes. GNNs have demonstrated remarkable performance in various tasks, including node classification, link prediction, graph classifications, and graph reconstruction~\cite{gcn2017,graphsage,gat2018,klicpera2018predict,ma2019memory,morris2019weisfeiler,zhang2022costa,zhang2018link,zhang2022graph,FeatureNorm2020,song2022towards,song2021semi,li2022bsal,zhang2022spectral,liu2022cspm,fu2020magnn}. They have also found wide applications in recommender systems, anomaly detection, social networks analysis, and more \cite{chen2021attentive,chen2023bipartite,chen2022learning,chen2022modeling,chen2023wsfe,dou2020enhancing,ma2021comprehensive,gcn2017,song2022towards,song2022hierarchical,song2021graph,yang2022htgn}. The majority of GNNs learn graphs in Euclidean space due to their computational advantages and intuitiveness.
However, Euclidean models are limited in their ability to represent complex patterns in graph~\cite{non_linear_embedding}.

\subsection{Hyperbolic Geometry}
Hyperbolic representation learning has recently garnered considerable attention~\cite{zhou2022hyperbolic,choudhary2022hyperbolic}.
Hyperbolic geometry has been recognized as a continuous tree~\cite{krioukov2009curvature}, exhibiting properties such as low distortion and small generalization errors when modeling tree-like structured data~\cite{sarkar2011low,suzuki2021generalization1,suzuki2021generalization2}.
Its applications span various domains~\cite{yang2022hyperbolic,mettes2023hyperbolic,peng2021hyperbolic}, including computer vision~\cite{khrulkov2020hyperbolic,atigh2022hyperbolic,hsu2021capturing}, natural language processing~\cite{nickel2017poincare,nickel2018learning,sala2018representation,montella2021hyperbolic,kolyvakis2019hyperkg,bai2021modeling,chami2020low}, recommender systems~\cite{HyperML2020,yang2022hrcf,sun2021hgcf,wang2021hypersorec,yang2022hicf,chen2022modeling}, graph learning~\cite{gulcehre2018hyperbolic,zhang2021hyperbolic,hgcn2019,liu2019HGNN,liu2022enhancing,yang2021discrete,bai2023hgwavenet} and more~\cite{xiong2022hyperbolic}. In the graph learning domain, recent works~\cite{liu2019HGNN,hgcn2019,lou2020differentiating,lgcn,zhang2021hyperbolic,yang2022htgn,yang2021discrete,liu2022enhancing} have generalized graph neural networks to hyperbolic space and demonstrated impressive performance, particularly on tree-like data. Some studies~\cite{peng2020mix,wang2021mixed,zhu2020gil,sun2021self} propose learning graph in different embedding spaces or product spaces. 
Furthermore, researchers have also explored the use of ultrahyperbolic geometry for graph learning~\cite{xiong2022ultrahyperbolic,xiong2021pseudo,law2020ultrahyperbolic,law2021ultrahyperbolic}. However, many existing methods fail to consider the local heterogeneous structure of graphs, resulting in significant distortion and low-quality embeddings.

    \begin{figure*}[ht]
    \centering
    \includegraphics[width=0.82\textwidth]{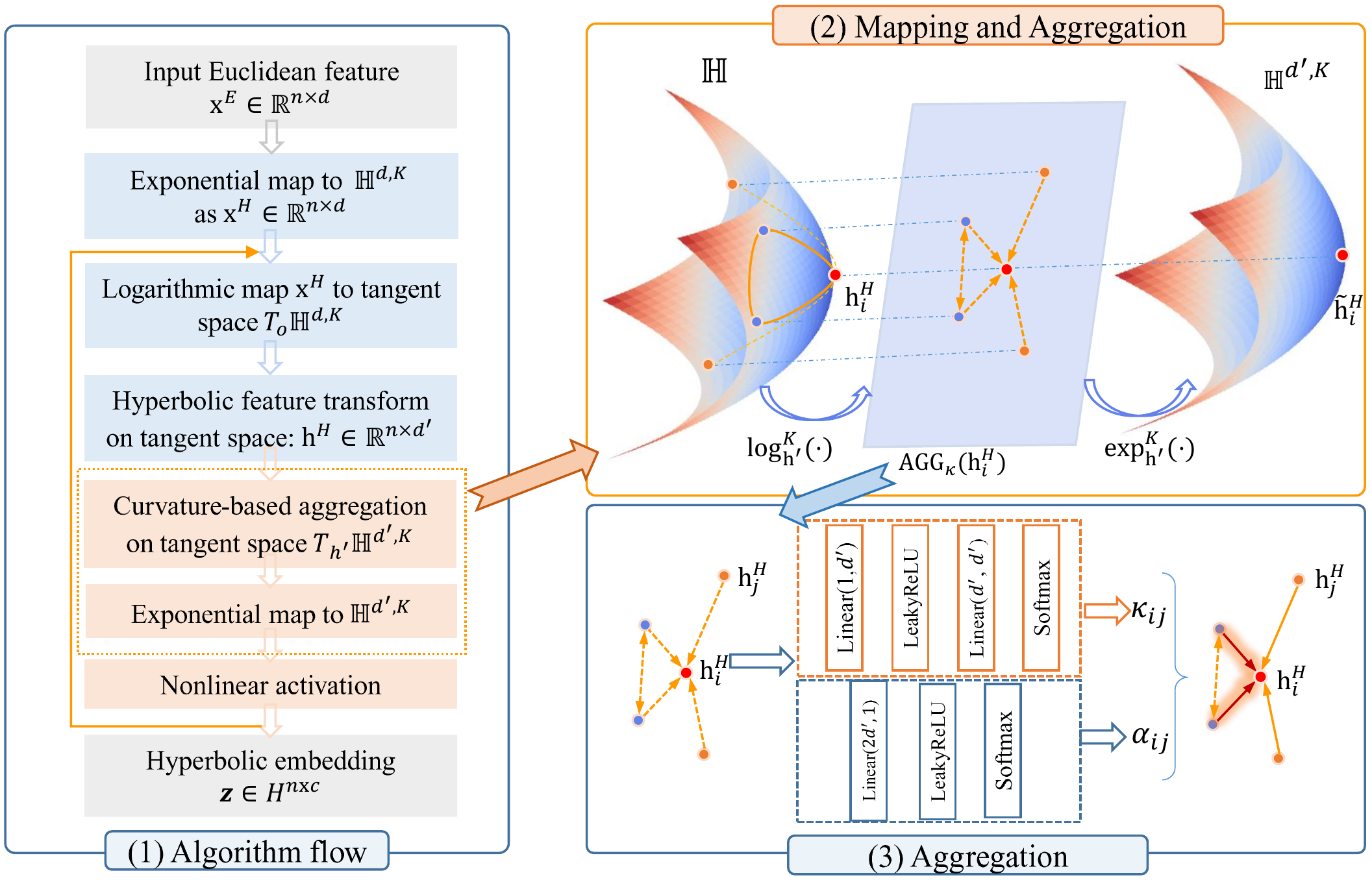}
    \small
    \caption{ Schematic of \method. (1) The simplified algorithm flow of our method: consists of hyperbolic projection, feature transform, and aggregation. After that, a readout layer is applied to the embeddings for either a node classification or link prediction task. (2) The visualization of neighborhood aggregation procedure: first project information to hyperbolic space for transformation, then map messages to the tangent space, perform the aggregation in the tangent space with the guide of discrete curvature (and attention), and then map back to the hyperbolic space. (3) The details of Ricci Curvature-aware aggregation and its combination with feature-based attention.}
    \label{fig.framework}
    \end{figure*}
    
\subsection{Graph Curvature}

    Graph curvature, resembling curvature in continuous geometry, includes Gaussian curvature, Ricci curvature, and average curvature. Each of these elements serves a distinct purpose: Gaussian curvature measures local curvature at vertices, Ricci curvature assigns curvature to edges, and average curvature offers a global measure for the entire graph. Applications of graph curvatures span various domains in network alignment, congestion and vulnerability detection, community detection, and robustness analysis~\cite{ni2018network,ni2015ricci,jonckheere2011congestion}. 
    The recent work of curvature graph neural network (CurvGN)~\cite{ye2019curvature} introduced the notion of Ricci curvature into the field of graph learning. 
    The study in~\cite{topping2021understanding} showed that edges with negative curvature can contribute to the over-squashing problem in graph embeddings. 
    Coinciding with the announcement of the accepted papers for WWW 2023, we noted a parallel work by Fu et al.~\cite{fu2023hyperbolic} that introduces the idea of class-aware Ricci curvature for addressing hierarchy-imbalance in node classification, while in our work, we aim to explore the integration of more generalized ricci curvature with hyperbolic graph convolution and curvature-based filtering mechanism to enhance the performance of HGCN for a more range of tasks, including node classification and link prediction.

\section{Background}
    In this part, we first briefly review the necessary definitions of differential geometry, primarily concentrating on hyperbolic geometry. A thorough and in-depth explanation can be found in~\cite{lee2018introduction}. We also give a short introduction about Ollivier Ricci curvature (ORC), which is a generalized Ricci curvature tailored for discrete objects (e.g., graphs)~\footnote{Our work is also applicable to other types of Ricci curvature.}. The readers may refer to~\cite{ollivier2009ricci} for more details.
    
	\subsection{Riemannian Geometry}
	\textbf{Manifold and Tangent Space.} Riemannian geometry, a subfield of differential geometry, denoted as $\mathcal{M}$ with a Riemannian metric $g$. An $n$-dimensional manifold $(\mathcal{M},g)$ represents a smooth and real space, essentially an extension of a 2-D surface to higher dimensions, that can be locally approximated by $\mathbb{R}^n$. For any point $\mathbf{x}$ on $\mathcal{M}$, we define a tangent space $\mathcal{T}_\mathbf{x}\mathcal{M}$, acting as the first-order approximation of $\mathcal{M}$ in the vicinity of $\mathbf{x}$. This tangent space is an $n$-dimensional vector space that is isomorphic to $\mathbb{R}^n$. The metric $g$ on the Riemannian manifold designates a smoothly changing positive definite inner product $<\cdot,\cdot>:\mathcal{T}_\mathbf{x}\mathcal{M}\times \mathcal{T}_\mathbf{x}\mathcal{M}\to \mathbb{R}$ on this tangent space, thereby facilitating the definition of numerous geometric attributes, such as geodesic distances, angles, and curvature. 

    \textbf{Geodesics and Induced Distance Function.} For a curve $\gamma:[\alpha,\beta]\to \mathcal{M}$, the shortest length of $\gamma$, i.e., geodesics, is defined as $L(\gamma)=\int_\alpha^\beta\|\gamma^\prime(t)\|_g dt$. Then the distance of $\mathbf{u}, \mathbf{v} \in \mathcal{M}$, $d^\mathcal{M}(\mathbf{u},\mathbf{v})=\inf L(\gamma)$ where $\gamma$ is a curve that $\gamma(a)=\mathbf{u}, \gamma(b)=\mathbf{v}$. 
    
    \textbf{Maps and Parallel Transport.}
    The maps define the relationship between the hyperbolic space and the corresponding tangent space. 
    Given a point $\mathbf{x}\in \mathcal{M}$ and a vector $\mathbf{v}\in\mathcal{T}_\mathbf{x}\mathcal{M}$, a unique geodesic $\gamma:[0,1]\to\mathcal{M}$ exists, satisfying $\gamma(0)=\mathbf{x}, \gamma^\prime(0)=\mathbf{v}$. The exponential map, symbolized as $\exp_\mathbf{x}: \mathcal{T}_\mathbf{x}\mathcal{M} \to \mathcal{M}$, is defined such that $\exp_\mathbf{x}(\mathbf{v})=\gamma(1)$. Conversely, the logarithmic map, denoted as $\log_\mathbf{x}$, acts as the inverse of $\exp_\mathbf{x}$. Furthermore, the parallel transport $PT_{\mathbf{x}\rightarrow \mathbf{y}}:\mathcal{T}\mathbf{x}\mathcal{M}\to\mathcal{T}\mathbf{y}\mathcal{M}$ achieves the transportation from point $\mathbf{x}$ to point $\mathbf{y}$, while ensuring the preservation of the metric tensors.
    
    \textbf{Hyperbolic Models.}
    Hyperbolic geometry describes a curved space with negative curvature. There are several mathematically equivalent ways to model hyperbolic geometry that emphasize different properties, but our methods apply to hyperbolic geometry in general and are not limited to any particular model. Formulas for concepts such as distance, maps, and parallel transport are summarized in Appendix~\ref{appendix:sec:hyperbolic_geometry}.

    \subsection{Graph Curvature}
    Curvature is a fundamental concept in smooth spaces that has also generalized to discrete objects (e.g., graphs). 
    There are several distinct notions of discrete Ricci curvature for graphs or networks, such as  the Forman-Ricci curvature~\cite{forman2003bochner} and Ollivier-Ricci curvature~\cite{ollivier2009ricci}. Here we mainly focus on ORC since it is more geometrical~\cite{ollivier2009ricci,jost2014Ricci_triangles}. Another reason is ORC builds a bridge between continuous geometry and discrete structures~\cite{van2020ollivier,ache2019ricci}.
    
\begin{definition}[Ollivier-Ricci Curvature]
\label{def:curvature}
Let $G=(V, E)$ be a locally finite, connected, and simple graph (i.e., $G$ contains no multiple edges or self-loops), for any two distinct vertices $v_1, v_2$, the ORC of $v_1$ and $v_2$ is defined as
    \begin{equation}
    \label{equ:curv_1}
        \kappa(v_1,v_2)=1-\frac{W(m_{v_1}, m_{v_2})}{d(v_1,v_2)} \in(-2, 1),
    \end{equation}
    where $d(v_1, v_2)$ is the shortest path between $v_1$ and $v_2$ on graph $G$, $W(m_{v_1}, m_{v_2})$ is the Wasserstein distance (see Definition~\ref{def:wasserstein's distance}) between two probability measures (see Definition~\ref{def:Probability_measure}) $m_{v_1}$ and $m_{v_2}$.
	\end{definition}
	\begin{definition}[Wasserstein Distance]
    \label{def:wasserstein's distance}
    Let $m_1, m_2$ be two probability measures on $V$. The \textit{Wasserstein distance} $W(m_1, m_2)$ between $m_1$ and $m_2$ is given by
    \begin{equation}
    W(m_1, m_2)=\inf_{\pi_{i,j}\in\Pi}\sum_{v_i, v_j\in V}\pi_{i,j}(v_i, v_j)d(v_i, v_j),
    \label{equ:curv_2}
    \end{equation}
    where $\pi_{i,j}:V\times V\to[0,1]$ is a transport plan, i.e., the probability measure of the amount of mass transferred from $v_i$ to $v_j$. Then to seek an optimal transference plan ($\pi$) that is to minimize the total cost of moving from $v_i$ to $v_j$ such that for every $v_i,v_j$ in $V$ satisfying
    \begin{equation}
        \sum_{v_i\in V}\pi_{i,j}(v_i, v_j)=m_1;
        \sum_{v_j\in V}\pi_{i,j}(v_i, v_j)=m_2.
    \label{equ:curv_3}
    \end{equation}
    \end{definition}
    \begin{definition}[Probability Measure]
	\label{def:Probability_measure}
	Given $G=(V,E)$, for a vertex $v_i$ in $V$, denote $d_{v_i}$ the degree of $v_i$ and $N(v)$ the neighbors of $v$, for any $p\in [0,1]$, the probability measure $m_{v_i}$ on $V$ is defined as:
	\begin{equation}
    m_{v_i}=\left\{\begin{array}{ll}
    p, & \text { if } v=v_i \\
    \frac{1-p}{d_{v}}, & \text { if }( v_i) \in N(v). \\
    0, & \text { otherwise }
    \end{array}\right.
    \label{equ:curv_4}
	\end{equation}
   \end{definition}
    \paragraph{{\textbf{\color{blue!50!black}Geometric Intuition}}} 
    ORC seeks the most efficient transportation plan that preserves mass between two probability measures, which may be solved using linear programming. Intuitively, transporting messages between two nodes whose neighbors are highly overlapping, such as two nodes in the same cluster, is costless. On the other hand, if two nodes are situated in distinct groups or clusters, information flow between them is difficult.
    
\section{Methodology}
The proposed method, \method, combines discrete and continuous curvatures to improve tree-like modeling in hyperbolic space. Our approach emphasizes the \textit{strengthening} of message passing in nodes with high local graph curvature and the \textit{weakening} of message propagation in nodes with low local curvature. This curvature-guided approach enhances the formation of hierarchies and reduces the impact of structural incompatibility on the modeling process.

\subsection{\method}
\label{sec:agg_method}
Our approach, named \method, presents a novel curvature-aware hyperbolic graph network model, as depicted in \fig~\ref{fig.framework}. Building upon the foundation of HGCN~\cite{hgcn2019}, we implement graph convolution operations via the tangential method~\cite{hgcn2019,liu2019HGNN} space. However, it is worth mentioning that \method is flexible and can be applied to non-tangential methods as well~\cite{lgcn}. Similar to other GNN and HGNN models, \method also comprises three fundamental modules: hyperbolic feature transformation, curvature-aware neighbor aggregation, and non-linear activation.


    
	\noindent
	\subsubsection{Hyperbolic Feature Transformation}
	{Hyperbolic feature transformation} is formulated as:
	\begin{equation}
	\mathbf{h}_{i}^{\ell, \mathcal{H}}=\left(\mathbf{W}^{\ell} \otimes^{\kappa_{\ell-1}} \mathbf{x}_{i}^{\ell-1, \mathcal{H}}\right) \oplus^{\kappa_{\ell-1}} \mathbf{b}^{\ell},
	\label{equ:transformation}
	\end{equation}
	where  $\ell$ denotes the  $\ell$-th layer, $\mathbf{W}$ is the trainable matrix and $\mathbf{b}$ is the bias. $\mathbf{W} \otimes^{\kappa} \mathbf{x}:=\exp _{\mathbf{o}}^{\kappa}(\mathbf{W} \log _{\mathbf{o}}^{\kappa}(\mathbf{x}))$ and $\mathbf{x} \oplus^{\kappa} \mathbf{b}:=\exp^\kappa_{\mathbf{x}}(P T_{\mathbf{o}\rightarrow\mathbf{x}}^{\kappa}(\mathbf{b}))$ are matrix-vector multiplication and bias translation operations in hyperbolic space, respectively. The superscript $^\mathcal{H}$ denotes the hyperbolic feature.

\begin{figure}[!tp]
\centering
\includegraphics[width=0.35\textwidth]{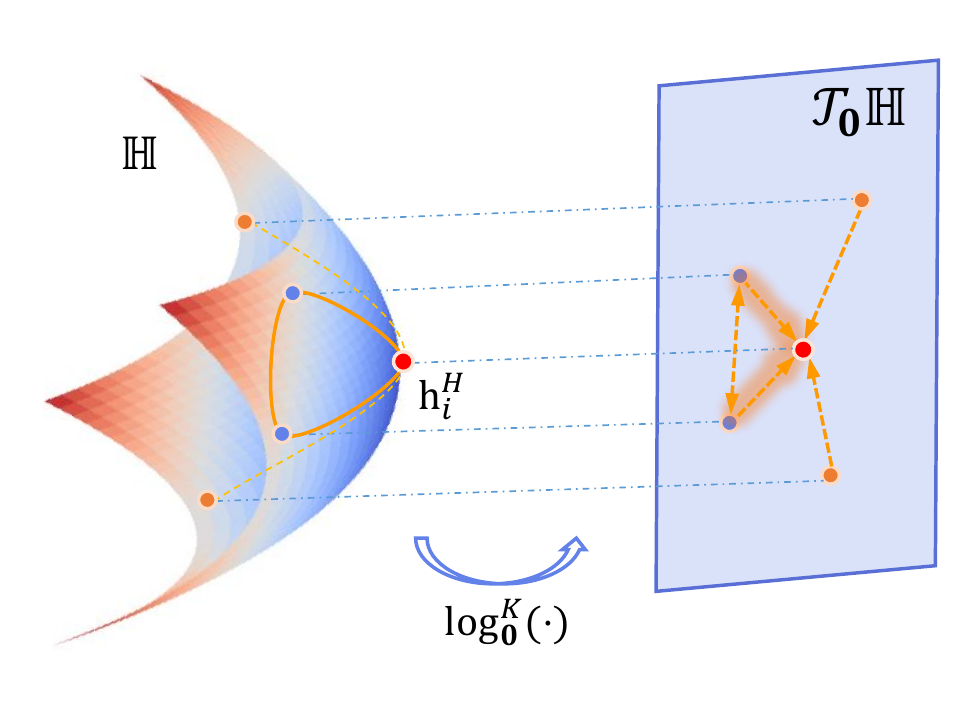}
\caption{Illustration of the proposed Curvature-aware aggregation. Compared with the traditional aggregation scheme, the proposed method further makes the aggregation process aware of the local structure.}
\label{fig:aggregation}
\vspace{-10pt}
\end{figure}

	\subsubsection{Curvature-Aware Neighbor Aggregation}
	{Curvature-Aware Neighbor Aggregation, as shown in {\color{black}Figure~\ref{fig:aggregation}}, is built in the tangent space of the origin:
	\begin{equation}
	{\tilde{\mathbf{h}}_{i}^{\ell, \mathcal{H}} 
	=\exp_{\mathbf{o}}^{\kappa_{l-1}}\left(\sum_{j\in \mathcal{N}_i}{\tilde{\kappa}_{i,j} \cdot\log_{\mathbf{o}}^{\kappa_{l-1}}(\mathbf{h}_j^{\ell,\mathcal{H}})}\right)}.
	\label{equ:kappa_aggregation}
	\end{equation}
	Here $\tilde{\kappa}_{i,j}$ denotes the curvature for capturing the local structure, computed by softmax operation within the neighbors $\mathcal{N}_i$ ($\mathcal{N}_i$ contains the node self):
	\begin{equation}
	\label{equ:softmax_kappa}
	\tilde{\kappa}_{i,j} =\textrm{softmax}_{j\in\mathcal{N}(i)}\left(\textrm{MLP}(\kappa_{i,j})\right),
	\end{equation}
    where $\kappa_{i,j}$ is the raw ORC value, and MLP (Multilayer Perceptron) is employed to make the curvature more adaptive to the overall negative curvature. This approach is referred to as {\sc Curv}.
    When it comes to the case where the topology information and node features are inconsistent to a certain degree, e.g., the network is quite sparse or depends more on the node features, inspired by~\cite{zhang2020adaptive}, we propose a feature-attention enhanced aggregation ({\sc CurvAtt}), which encodes node state into the curvature:
	\begin{equation}
	\label{equ:feature-enhanced aggregation}
	    	\tilde{\kappa}'_{i,j}=	\frac{w_{\kappa}\tilde{\kappa}_{i,j}+w_\alpha\alpha_{i,j}}{w_{\kappa}+w_\alpha},
	\end{equation}
	where $\alpha_{i,j}$ is hyperbolic feature attention, $w_{\kappa}$ and $w_{\alpha}$ are the trainable parameters that adjust structure information and feature correlation with the initial value $1.0$.
	The hyperbolic feature attention $\alpha_{i,j}$ is defined as:
	\begin{equation}
	{\alpha}_{i,j} =\mathrm{softmax}_{j\in\mathcal{N}_i}\left(\textrm{MLP}(\log_{\mathbf{o}}^\kappa(\mathbf{h}_i^\mathcal{H})\|\log_{\mathbf{o}}^\kappa(\mathbf{h}_j^\mathcal{H}))\right).
	\end{equation}
	\subsubsection{Hyperbolic Non-linear Activation}
	After that, we apply a nonlinear activation:
    \begin{equation}
	\mathbf{x}_{i}^{\ell, \mathcal{H}}=\sigma^{{\kappa_{\ell-1}, \kappa_{\ell}}}(\tilde{\mathbf{h}}^{\ell, \mathcal{H}})=\exp _{\mathbf{o}}^{\kappa_{\ell}}\left(\sigma(\log _{\mathbf{o}}^{\kappa_{\ell-1}}(\tilde{\mathbf{h}}^{\ell, \mathcal{H}}))\right).
	\label{equ:non_linear}
	\end{equation}


    \paragraph{{\textbf{\color{blue!50!black}Geometric Intuition}}}
    The real-world tree-like graphs with heterogeneous local structures are inevitably distorted if we directly embed them into a homogeneous manifold.
    For instance, the embedding of quasi-cycle graphs such as $n\times n$ square lattices (zero curvature) and $n$-node cycles (positive curvature) incur at least a multiplicative distortion of $O(n/\log n)$ in hyperbolic space~\cite{hyperbolic_distortion}. {Graph Ricci curvature is able to mitigate this distortion}. 
    The geometric intuition is that the more positive the curvature is, the more two distance balls centered at nearby points overlap, and therefore, the cheaper it is to transport the mass from one to the other. 
    Theoretical results show that with the increasing number of triangles involved in the linked pair $(i, j)$, the lower bound of curvature will be increased~\cite{jost2014Ricci_triangles}, as stated in Theorem \ref{theorem:lower_bound of ORC}.
    It is easy to understand because when the two vertices share many triangles, then the transportation distance should be smaller, and the curvature, therefore, is correspondingly larger. 
    \begin{theorem}[Lower bound of ORC~\cite{jost2014Ricci_triangles}]
        On a locally finite graph, for any pair of neighboring vertices i, j,
        let $\#(i, j):=$ number of triangles which include $i, j$ as vertices for $i \sim j$. Then, we have the inequality, saying that
        \begin{equation}
        \begin{aligned}
        \kappa_{i, j} \geq
        &-\left(1-\frac{1}{d_{i}}-\frac{1}{d_{j}}-\frac{\#(i, j)}{d_{i} \wedge d_{j}}\right)_{+} \\
        &-\left(1-\frac{1}{d_{i}}-\frac{1}{d_{j}}-\frac{\#(i, j)}{d_{i} \vee d_{j}}\right)_{+} 
        +\frac{\#(i, j)}{d_{i} \vee d_{j}},
        \end{aligned}
        \end{equation}
        where $s_{+}:= \max(s,0), s\vee t:= \max(s,t), \quad and\quad s\wedge t:= \min(s,t).$
        \label{theorem:lower_bound of ORC}
    \end{theorem}
        In this study, we make a theoretical analysis in Theorem~\ref{prop:ORC_embedding_distance}, which further demonstrates the relations of ORC and embedding distance, i.e., when a large curvature is involved within the linked node, the closer of their embedding distance, which thus mitigates the distortion.
    \begin{theorem}[Embedding Distance w.r.t ORC]
    \label{prop:relations of ORC}
        Let $(i, j)\in E$ be the linked pair, $\mathbf{h}_i, \mathbf{h}_j\in \mathbb{R}^d$ be the node state in the tangent space, $d_i$ be the degree of node $i$, and $D$ be the distance of node $i$ and node $j$ in the tangent space, that is
            \begin{equation}
                {D} = \|\mathbf{h}_i - \mathbf{h}_j\|,
            \end{equation}
        where $\|\cdot\|$ is the Euclidean norm. Define a large ORC as $\tilde{\kappa}_{i, j}>\max\{1/d_i, 1/d_j\}$ and a small ORC as $\tilde{\kappa}_{i, j}<\min\{1/d_i, 1/d_j\}$. Then, when the large ORC is involved, their embedding distance will get smaller if using \method. On the contrary, when the small ORC is involved, their embedding distance will get larger.
        \label{prop:ORC_embedding_distance}
    \end{theorem}
    \begin{proof}
    In the following, we use ${D}_l$ and ${D}_s$ to denote the distance when large and small curvature $\tilde{\kappa}$ are involved, respectively. The main idea is that when there is a large curvature involved, the node distance will be decreased compared with the original case (degree-based aggregation), that is ${D}_l < {D}$. At the same time, when there is a small curvature involved, the node distance will increase, that is ${D}_s > {D}$.
    
    (1) When a large curvature (i.e., $\kappa_{i,j}>\max(1/d_{i},1/d_{j})$) is involved, more messages will be transferred, and we decompose the embedding, taking $\mathbf{h}_i$ as an example, into two components: one is from original $\mathbf{h}_i$ and another is the incremental parts from $\mathbf{h}_j$, then
    \begin{equation}
    \begin{aligned}
        {D}_l&=\|(\mathbf{h}_i+\alpha_{i}\mathbf{h}_j)-(\mathbf{h}_j+\alpha_{j}\mathbf{h}_i)\| \\
        &=\|(\mathbf{h}_i-\alpha_{j}\mathbf{h}_i)
        -(\mathbf{h}_j-\alpha_{i}\mathbf{h}_j)\|,
    \end{aligned}
    \end{equation}
    where $\alpha_{i}$ ($\alpha_{j}$) is the difference between $\tilde{\kappa}_{i,j}$ and $1/d_i$ ( $1/d_j$), i.e., $\alpha_{i}=\tilde{\kappa}_{i,j}-1/d_i$, $\alpha_{j}=\tilde{\kappa}_{i,j}-1/d_j$. Since $\tilde{\kappa}_{i,j}> \max(1/d_{i},1/d_{j})$, $\alpha_i$ and $\alpha_j$ are both positive. Let $\alpha_{ij}=\alpha_{i}\approx\alpha_{j}$, then
    \begin{equation}
    \begin{aligned}
        {D}_l&\approx\|(\mathbf{h}_i-\alpha_{ij}\mathbf{h}_i)
        -(\mathbf{h}_j-\alpha_{ij}\mathbf{h}_j)\|\\
        &=(1-\alpha_{ij})\|\mathbf{h}_i-\mathbf{h}_j\| \\
        &< {D}.
    \end{aligned}
    \end{equation}
    Then, we easily know that when large curvature is involved, the distance will be reduced and two nodes will be closer to each other. What's more, the larger the curvature, the closer the nodes are.
    
    (2) Similarly, when a small curvature ($\kappa_{i,j}<\min(1/d_{i},1/d_{j})$) is involved, fewer messages will be transferred, and we decompose the embedding, taking $\mathbf{h}_i$ as an example, into two components: one is from original $\mathbf{h}_i$ and another is the reduction parts of $\mathbf{h}_j$, that is
    \begin{equation}
    \begin{aligned}
        {D}_s&=\|(\mathbf{h}_i-\beta_{i}\mathbf{h}_j)-(\mathbf{h}_j-\beta_{j}\mathbf{h}_i)\| \\
        &\approx(1+\beta_{ij})\|\mathbf{h}_i-\mathbf{h}_j\| \\
        & >{D},
    \end{aligned}
    \end{equation}
    where $\beta_{i}$ is the difference between $1/d_i$ and $\tilde{\kappa}_{i,j}$, i.e., $\beta_{i}=1/d_i-\tilde{\kappa}_{i,j}$, $\beta_{j}=1/d_j-\tilde{\kappa}_{i,j}$. Both $\beta_i$ and $\beta_j$ are positive in that $\kappa_{i,j}<\min(1/d_{i},1/d_{j})$. Let  $\beta_{ij}=\beta_{i}\approx\beta_{j}$, then, we easily know that when small curvature is involved, the node pair will become more distant in the embedding space. What's more, the smaller the curvature, the more distant the nodes are.
    \end{proof}
    
    \subsection{Curvature-based Homophily Constraint}
        In the degree-based learning paradigm, like GCN~\cite{gcn2017}, the influence of a node on another node decays exponentially as their graph distance increases as shown by~\cite{subgraph_meta_learning}. The analysis in \cite{subgraph_meta_learning} is limited degree-based aggregation and Euclidean space. The hyperbolic message passing learning paradigm of \method also shows a similar phenomenon, which causes too much influence loss in long-term propagation. 
        Especially, if the paths consist of numerous connections to other nodes, the node influence is minimal. For clarity, we term the hyperbolic message passing learning paradigm in \method or original HGNNs~\cite{liu2019HGNN,hgcn2019,zhang2021hyperbolic} as \aggmethod. The \aggmethod is a local aggregation method, in which the influence of nodes decreases with increasing distance, as demonstrated in Theorem~\ref{theorem:decay}.
        
        \begin{theorem}[Decaying property of \aggmethod]
        \label{theorem:decay}
        Let $p$ be a path between node ${u}$ and node ${v}$, $d_{g}^*$ be the shortest distance between ${u}$ and ${v}$, let $C$ be a constant and $\mathbf{z}$ be the embedding on the tangent space. Consider the node influence $I_{{u},{v}}$ ($I_{u,v}=\|{\partial {\mathbf{z}_u}}{/\partial {\mathbf{z}_v}}\|$) from ${v}$ to ${u}$ using \aggmethod, $I_{{u}, {v}}\leq C\gamma^{d_{g}^*}(0<\gamma<=1)$. The condition for equality is $d_g^*=1$, and ${v}$ is the unique neighbors of node ${u}$, correspondingly $C=1, \gamma=1$.
        \end{theorem}
        \begin{proof}
        Recall the aggregation rule in Equations~(\ref{equ:kappa_aggregation}) and (\ref{equ:softmax_kappa}) (similar to that in origin HGNNs), we focus on the aggregation in the tangent space and ignore the previous logarithmic map and the later exponential map since they are applied before and after the whole aggregation process, respectively. 
        Then for any node ${u}$ and ${v}$, the update rule in the tangent space can be formulated as:
        \begin{equation}
            \mathbf{z}_u= \sum_{j\in\mathcal{N}(u)}{\tilde{\kappa}_{u,j}\mathbf{z}_j} =\frac{1}{K_{uu}}\sum_{j\in\mathcal{N}(i)}\exp(\tilde{\kappa}_{u,j})\mathbf{z}_j,
        \end{equation}
        where 
        $K_{uu}={\sum_{j\in\mathcal{N}(u)}{\exp(\tilde{\kappa}_{u,j})}}.$\footnote{For original HGNNs, the $\kappa$ can be replaced with degree-based weight or attention-based weight.}
        By an expansion of node in the neighbor $\mathcal{N}(j)$, we have:
        \begin{equation}
            \mathbf{z}_u =\frac{1}{K_{uu}}\sum_{j\in\mathcal{N}(u)}\exp(\tilde{\kappa}_{u,j})*\frac{1}{K_{jj}}\sum_{k\in\mathcal{N}(j)}\exp(\tilde{\kappa}_{j,k}){\mathbf{z}}_k.
        \end{equation}
        We completely expand it:
        \begin{equation}
          \label{equ:z_u}
        \begin{aligned}
            \mathbf{z}_u =&\frac{1}{K_{uu}}\sum_{j\in\mathcal{N}(u)}\exp(\tilde{\kappa}_{i,j})*\cdots*\frac{1}{K_{oo}}\sum_{p\in\mathcal{N}(o)}\exp(\tilde{\kappa}_{o,p}){\mathbf{z}}_p.
        \end{aligned}
        \end{equation}
        Node influences $I_{u,v}$ of $v$ on $u$ in the message passing output is $I_{u,v}=\|{\partial {\mathbf{z}_u}}{/\partial {\mathbf{z}_v}}\|$, where the norm is any subordinate norm and the node influence measures how a change in $v$ passes to a change in $u$.
        By equation~(\ref{equ:z_u}), the node influence can be computed as:
        \begin{equation}
        \begin{aligned}
        &I_{u,v}=\|\frac{\partial \mathbf{z}_u}{\partial \mathbf{z}_v}\| \\
            &=\|\frac{\partial}{\partial \mathbf{z}_v}(\frac{1}{K_{uu}}\sum_{j\in\mathcal{N}(u)}\exp(\tilde{\kappa}_{i,j})*\cdots*{\frac{1}{K_{oo}}}\sum_{p\in\mathcal{N}(o)}\exp(\tilde{\kappa}_{o,p})\mathbf{z}_p)\|.
        \label{equ:partial_devs}
        \end{aligned}
        \end{equation}
        The partial derivative of the nodes in Equation ~(\ref{equ:partial_devs}) is zero if they are not on the path between node $u$ and $v$, and then the feature influence can be decomposed into the sum influence of all related paths. Suppose there are $n$ paths between $u$ and $v$, then
        \begin{equation}
        \label{equ:223}
        \begin{aligned}
            I_{u,v}=\left\|\frac{\partial}{\partial \mathbf{z}_v}\left( I_{p_1} + \cdots + I_{p_i} \cdots+ I_{p_n}\right)\right\|,
        \end{aligned}
        \end{equation}
        where 
        $$
        I_{p_i}=\underbrace{\frac{1}{K_{uu}}\exp(\tilde{\kappa}_{u,p_j^i})\cdots \frac{1}{K_{p_{n_i}^ip_{n_i}^i}}\exp(\tilde{\kappa}_{p_{n_i}^i,v})}_{S(I_{p_i})}\mathbf{z}_v.
        $$
        Note that, in Equation~(\ref{equ:223}), the scalar term $S(I_{p_i})$ ranges from $(0, 1]$ and all $I_{p_i}(1\leq i\leq n)$ have the term $m_v$, thus we separate $S(I_{p_i})$ and the rest derivative term and then uses the absolute homogeneous property, i.e., $\|\alpha M\|=|\alpha |\|M\|$
        \begin{equation}
        \begin{aligned}
            I_{u,v}
            &=\left|S(I_{p_1}) + \cdots + S(I_{p_i}) \cdots+ S(I_{p_n})\right|\|\frac{\partial z_v}{\partial z_v}\| \\
            &= \left|S(I_{p_1}) + \cdots + S(I_{p_i}) \cdots+ S(I_{p_n})\right| \\
            &\leq |n*\max(S(I_{p_i}))| \\
            &=|n*\gamma^{n^i}|\\
            &\leq |n*\gamma^{d_g^*}|\\
            &=C\gamma^{d_g^*},
        \end{aligned}
        \label{equ:23}
        \end{equation}
        where $d_g$ is the shortest path between $u$ and $v$, $d^*\leq n^i$ and $0<\gamma\leq 1$, thus the second inequality holds on in Equation~(\ref{equ:23}). For more generality, we use constant $C$ to denote the $n$. The condition for equality is if and only if $d_g^*=1$ and the $v$ is the unique neighbor of node $u$, i.e., $\gamma=1$ and $C=1$.
        \end{proof}

    \begin{table}[]
    \centering
    \caption{Comparisons of the abilities of models in terms of global tree-likeness modeling (Global), local heterogeneous structure learning (Local), and neighbor interaction (Neighbor) are indicated by $\checkmark$ for the presence of the ability and $\times$ for its absence.}
    \resizebox{\columnwidth}{!}{%
    \begin{tabular}{@{}lcccc@{}}
    \toprule
    Model Type                                    & Models   & Global & Local & Neighbor \\ \midrule
    \multirow{2}{*}{Shallow models}         & EUC      &  $\times$      &   $\times$    &   $\times$       \\
                                            & HYP      &   $\checkmark$     &    $\times$   &   $\times$       \\ \midrule
    \multirow{4}{*}{Euclidean GNN models}             & GCN      & $\times$      & $\times$     &    $\checkmark$      \\
                                            & GAT      & $\times$     &  $\times$     &     $\checkmark$     \\
                                            & SAGE     & $\times$    &   $\times$    &    $\checkmark$      \\
                                            & SGC      &  $\times$      &   $\times$    &  $\checkmark$        \\ \midrule
    \multirow{2}{*}{Curvature GNN models} & CurvGN   &    $\times$    &  $\checkmark$     &     $\checkmark$     \\
                                            & $\kappa$GCN &   $\checkmark$     &   $\times$    &    $\checkmark$      \\ \midrule
    \multirow{2}{*}{Hyperbolic GNN models}        & HGCN     &  $\checkmark$      & $\times$     &    $\checkmark$      \\
                                            & LGCN     &  $\checkmark$      & $\times$     &    $\checkmark$          \\ \midrule
            Curvature-aware HGNN model    & $\kappa$HGCN &  $\checkmark$      & $\checkmark$     &    $\checkmark$      \\ \bottomrule   
    \end{tabular}%
    }
    \label{tab:models-function}
    \end{table}
   
        Theorem~\ref{theorem:decay} shows that the node influence using \aggmethod exponentially decays as the shortest graph distance $d_g^*$ between two nodes increases. In other words, distant nodes in dense areas will have less interaction, even if they are in a dense connected area.
        To alleviate the phenomenon, we propose a Curvature-based Homophily Constraint (\hr) to enhance the connection within linked pairs. The basic idea is to push the embeddings of linked nodes closer if their ORC value is larger than a threshold. In this way, we can enforce disjoint node pairs in dense areas or clusters to have more influence on each other through their mutual neighbors, which is given by:
	\begin{equation} 
	\label{HGCN: lp_loss_fun}
    	\mathcal{L}_{\kappa hc_{+}} = -\frac{1}{|E_\kappa|}\sum_{(i,j)\in E_\kappa}\mathrm{log}~p(\mathbf{x}_{i}^{\ell, \mathcal{H}}, \mathbf{x}_{j}^{\ell, \mathcal{H}}),
	\end{equation}
	where $E_\kappa$ is the filtered edge set based on ORC threshold $\tau$\footnote{We select edges if their ORC value is larger than a threshold where edges can be constructed by multiple hop neighbors and the weight is added their curvature together based on their shortest distance.}; $p(\cdot)$ is the Fermi-Dirac function, indicating the probability of two hyperbolic nodes $(u, v)$ link or not, which is given by:
	\begin{equation}
	\label{equ:femi-dirac}
	p(\mathbf{x}_u, \mathbf{x}_v)=\left[\exp{(d_\mathcal{H}^2(\mathbf{x}_u, \mathbf{x}_v)-r) / t}+1\right]^{-1},
	\end{equation}
	and $d_\mathcal{H}(\mathbf{x}_u, \mathbf{x}_v)$ is the hyperbolic distance from $u$ to ${v}$, $r$ and $t$ is hyper parameters and we set it as previous work~\cite{hgcn2019}.  We also sample the same number of negative link pairs that they have no connections or the curvature is very small based on the results in~\cite{topping2021understanding}. Totally, 
 \begin{equation}
          \mathcal{L}_{\kappa hc}=\mathcal{L}_{\kappa hc_{+}} + \mathcal{L}_{\kappa hc_{-}}.
 \end{equation}
   

        \paragraph{\textbf{\color{blue!50!black}Geometric Intuition}} \aggmethod helps build the connection between the graph topology and the embedding space, adjust the curvature of the hyperbolic geometry, and guide the information flow. It also shortens the distance of two linked nodes in an area with many triangles, helping mitigate the distortion caused by hyperbolic space. Nonetheless, \aggmethod is local inherently, and the proposed \hr further enhances the interactions of unconnected nodes, which is non-local.



    \subsection{\method Architecture}
    Given the Euclidean feature $\mathbf{x}^E$, we first project it into the hyperbolic manifold by the exponential map. \method architecture takes layers of \aggmethod as the encoder.  Following the literature, the Fermi-Dirac function is used as a decoder in the link prediction task.  For the node classification task, the final hyperbolic vector is mapped back to tangent space and decoded with MLP, which is the same with work~\cite{hgcn2019}.

 \begin{table}[h]
\centering
\caption{Statistics of the datasets}
\resizebox{0.45\textwidth}{!}{%
\begin{tabular}{@{}llcccc@{}}
\toprule
{\sc Dataset} & Nodes & Edges & Classes & Node features & Hyperbolicity$\delta$ \\ \midrule
{\sc Disease (NC)} & 1044  & 1043  & 2       & 1000          & 0             \\
{\sc Disease (LP)} & 2665  & 2664  & 2       & 1000          & 0             \\
{\sc Airport} & 3188  & 18631 & 4       & 4             & 1             \\
{\sc PubMed}  & 19717 & 88651 & 3       & 500           & 3.5           \\
{\sc Cora}    & 2708  & 5429  & 7       & 1433          & 11            \\ \bottomrule
\end{tabular}
}
\label{tab:statistics} 
\end{table}

\section{Experiments}
\begin{table*}[]
\centering
\begin{minipage}[h]{0.48\textwidth}
\caption{Profiling evaluation on node classification. F1-score with standard deviation for {\sc Disease} and {\sc Airport}; accuracy for others (the higher, the better).}
\resizebox{1.0\textwidth}{!}{%
\begin{tabular}{@{}lrrrr@{}}
\toprule
{\sc Dataset}                          & {\sc Disease} & {\sc Airport} & {\sc PubMed} & {\sc Cora} \\
Hyperbolocity ($\delta$)                                                & 0                              & 1                              & 3.5                           & 11                          \\ \toprule
EUC                                                     & $32.5\pm1.1$                   & $60.9\pm3.4$                   & $48.2\pm0.7$                  & $23.8\pm0.7$                \\
HYP~\cite{nickel2017poincare}                                                     & $45.5\pm3.3$                   & $70.2\pm0.1$                   & $68.5\pm0.3$                  & $22.0\pm1.5$                \\ \midrule
GCN~\cite{gcn2017}                                                     & $69.7\pm0.4$                   & $81.4\pm0.6$                   & $78.1\pm0.2$                  & $81.3\pm0.3$                \\
GAT~\cite{gat2018}                                                     & $70.4\pm0.4$                   & $81.5\pm0.3$                   & $79.0\pm0.3$                  & ${83.0}\pm0.7$       \\
SAGE~\cite{hamilton2017SAGE}                                                    & $69.1\pm0.6$                   & $82.1\pm0.5$                   & $77.4\pm2.2$                  & $77.9\pm2.4$                \\
SGC~\cite{wu2019simplifying}                                                     & $69.5\pm0.2$                   & $80.6\pm0.1$                   & $78.9\pm0.0$                  & $81.0\pm0.1$                \\ \midrule
CurvGN~\cite{ye2019curvature}                                                  & ${89.8}\pm2.9$                 & $84.7\pm1.5$                   & $78.3\pm0.3$                  & $82.0\pm0.9$                \\
$\kappa$GCN~\cite{bachmann2020constant}                                             & $82.1\pm1.1 $                  & $84.4\pm0.4$                           & $78.3\pm0.6$                  & $80.8\pm0.6$                \\ \midrule
HGCN~\cite{hgcn2019}                                                    & $74.5\pm0.9$                   & ${90.6}\pm0.2$                 & ${80.3}\pm0.3$                & $79.9\pm0.2$                \\
LGCN~\cite{lgcn}                                                    & $84.4\pm1.0 $                  & $ 90.9\pm1.0$                  & $78.8\pm0.5$                  & $\textbf{83.3}\pm0.5$                \\ \midrule
\textbf{\method(Ours)} & $\textbf{92.3}\pm1.4$          & $\textbf{92.8}\pm0.4$          & $\textbf{82.1}\pm0.5$         & ${82.5}\pm0.6$              \\ \midrule
$\Delta_E(\%)$                                          & $+31.1$                        & $+13.0$                        & $+3.9$                        & $-0.6$                      \\
$\Delta_\kappa(\%)$                                     & $+2.8$                         & $+9.6$                         & $+4.9$                        & $+0.6$                      \\
$\Delta_H(\%)$                                          & $+2.2$                         & $+3.6$                         & $+5.0$                        & $+5.4$                      \\ \bottomrule
\end{tabular}
}
\label{tab:nc}
\vspace{-10pt}
\end{minipage}
\hspace{10pt}
\begin{minipage}[h]{0.48\textwidth}
\caption{Profiling evaluation on link prediction. AUC scores with standard deviation are reported (the higher, the better). The best is bold.}
\centering
\resizebox{1.0\textwidth}{!}{%
\begin{tabular}{@{}lrrrr@{}}
\toprule
{\sc Dataset}                          & {\sc Disease} & {\sc Airport} & {\sc PubMed} & {\sc Cora} \\
Hyperbolicity($\delta$)                                                & 0                              & 1                              & 3.5                           & 11                          \\ \midrule
EUC                                                     & $60.9\pm3.4$                   & $92.0\pm0.0$                   & $83.3\pm0.1$                  & $82.5\pm0.3$                \\
HYP~\cite{nickel2017poincare}                                                     & $70.2\pm0.1$                   & $94.5\pm0.0$                   & $87.5\pm0.1$                  & $87.6\pm0.2$                \\ \midrule
GCN~\cite{gcn2017}                                                     & $64.7\pm0.5$                   & $89.3\pm0.4$                   & $91.1\pm0.5$                  & $90.4\pm0.2$                \\
GAT~\cite{gat2018}                                                     & $69.8\pm0.3$                   & $90.5\pm0.3$                   & $91.2\pm0.1$                  & $93.7\pm0.1$                \\
SAGE~\cite{graphsage}                                                    & $65.9\pm0.3$                   & $90.4\pm0.5$                   & $86.2\pm1.0$                  & $85.5\pm0.6$                \\
SGC~\cite{wu2019simplifying}                                                     & $65.1\pm0.2$                   & $89.8\pm0.3$                   & $94.1\pm0.0$                  & $91.5\pm0.1$                \\ \midrule
CurvGN~\cite{ye2019curvature}                                                  & $80.6\pm0.8$                   & $89.5\pm0.3$                   & $91.6\pm0.4$                  & $72.5\pm0.7$                \\
$\kappa$GCN~\cite{bachmann2020constant}                                             & $92.0\pm0.5 $                  &  $92.5\pm0.5$                          & $94.9\pm0.3$                  & $92.6\pm0.4$                \\ \midrule
HGCN~\cite{hgcn2019}                                                    & $90.8\pm0.3$                   & $96.4\pm0.1$                   & ${96.3}\pm0.0$                & $92.9\pm0.1$                \\
LGCN~\cite{lgcn}                                                    & ${96.6}\pm0.6$                 & ${96.0}\pm0.6$                 & ${96.6}\pm0.1$                & ${93.6}\pm0.4$              \\
\midrule
\textbf{\method(Ours)} & {$\textbf{96.7}\pm0.1$}      & {$\textbf{98.2}\pm0.1$}      & {$\textbf{96.7}\pm0.1$}     & {$\textbf{95.0}\pm0.1$}   \\ \midrule
$\Delta_E(\%)$                                          & $+27.2$                        & $+8.5$                         & $+2.8$                        & $+1.4$                      \\
$\Delta_\kappa(\%)$                                     & $+4.1$                         & $+6.2$                         & $+1.9$                        & $+2.6$                      \\
$\Delta_H(\%)$                                          & $+0.6$                         & $+0.2$                         & $+1.3$                        & $+1.1$                      \\ \bottomrule
\end{tabular}
}
\label{tab:lp}
\vspace{-10pt}
\end{minipage}
\end{table*}  
\subsection{Experimental Setup}
    \textbf{Datasets.} 
    The evaluation of our work utilizes several datasets, including {\sc Disease}, {\sc Airport}, and three benchmark citation networks, namely {\sc PubMed} and {\sc Cora}. While {\sc Disease} and {\sc Airport} exhibit a more hierarchical structure, the citation networks are less so, making them suitable for demonstrating the generalization capability of our proposal. In Table~\ref{tab:statistics}, we provide the data statistics and hyperbolicity metric that measures the tree-likeness of each graph. For further details, please refer to Appendix~\ref{appendix:datasets}.
    
    \textbf{Baselines.} We compare our proposed model with various baselines.
    (1) \emph{Shallow Euclidean and hyperbolic models}, including Euclidean embeddings (EUC) and Poincar{\'e} embeddings (HYP)~\cite{nickel2017poincare};
    (2) \emph{Euclidean GNN models}, i.e., GCN~\cite{gcn2017}, GraphSAGE (SAGE)~\cite{graphsage}, Graph Attention Networks (GAT)~\cite{gat2018},	Simplified Graph Convolution, (SGC)~\cite{wu2019simplifying};
    (3) \emph{Curvature GNN models}, including Curvature Graph Network (CurvGN)~\cite{ye2019curvature} which applies the discrete curvature in Euclidean model and ProdGCN which deploys GNNs to products of constant curvature spaces, both of them are close to our work;
    \emph{Hyperbolic GNNs}, including HGCN~\cite{hgcn2019}, HGNN~\cite{liu2019HGNN}, HGAT~\cite{zhang2021hyperbolic}, LGCN~\cite{lgcn}. Table~\ref{tab:models-function} presents the different features of the aforementioned models regarding their capabilities for perceiving both global tree-likeness modeling and local heterogeneous structure, as well as their interactional aptitude with regard to neighboring information.

\textbf{Experimental Details}
\emph{Data split}. We evaluate \method on both node classification and link prediction tasks. The data split is the same with the previous works~\cite{hgcn2019}. More specifically, in link prediction, we randomly split edges into 85\%, 5\%, 10\% for training, validation, and test sets, respectively. For node classification, we split nodes into 70\%, 15\%, 15\% for {\sc Airport}, 30\%, 10\%, 60\% for {\sc Disease}, and we use 20 labeled train examples per class for {\sc Cora}, and {\sc PubMed}. 
\emph{Implementation details}. We closely follow the parameter settings as HGCN~\cite{hgcn2019}, fix the number of embedding dimensions to 16 and then perform hyper-parameter search on a validation set over learning rate, weight decay, dropout, and the number of layers. We also adopt the early stopping strategies based on the validation set as \cite{hgcn2019}. For baselines, we mainly refer to the reported results in the literature, and for the inconsistent cases (such as different embedding dimensions in H2H-GCN), we re-implement their official code in similar experimental settings.
\emph{Evaluation metric}.
Following the literature, we report the F1-score for {\sc Disease} and {\sc Airport} datasets, and accuracy for the others in the node classification tasks. For the link predictions task, the Area Under Curve (AUC) is calculated. 

\subsection{Experimental Results}

We report the results of 10 random experiments\footnote{The results on Lorentz model are similar.}, including standard deviations in \tab~\ref{tab:nc} and \ref{tab:lp}, where the $\Delta_E, \Delta_\kappa, \Delta_H$ is the improvement of the proposed model \method over the Euclidean GNNs, Curvature-related GNNs and Hyperbolic GNNs, respectively.

\paragraph{Node Classification} 
The experimental results of node classification are summarized in \tab \ref{tab:nc}, where a lower hyperbolicity value corresponds to a more tree-like structure. The key findings are:
(1) Overall, the proposed model performs impressively, surpassing previous models on four out of five datasets. Specifically, hyperbolic models (e.g., HGCN, LGCN) perform substantially better on the more hyperbolic dataset (e.g., {\sc Disease}) than on the less hyperbolic dataset; Euclidean models (e.g., GCN, GAT) find more success on the less hyperbolic datasets (e.g., {\sc Cora}) than on the more hyperbolic dataset; whereas our model performs better on both datasets, which is consistent with our motivation, namely, that the graph can be better learned under the guidance of curvature. In addition, from the improvements of $\Delta_\kappa$, we discovered that CurvGN with discrete curvature and $\kappa$GCN with continuous curvature both perform worse than our method which validates the power of hyperbolic geometry and the curvature-aware learning.
(3) When it comes to {\sc Cora}, both hyperbolic models and the proposed \method fail to outperform Euclidean GAT~\cite{gat2018}, indicating Euclidean geometry is more suitable for modeling data with scarcely hierarchical structures. Nevertheless, it is noted that \method still outperforms well-known Euclidean GCN models, e.g., GCN~\cite{gcn2017}, SGC~\cite{wu2019simplifying}, SAGE~\cite{graphsage}. It is observed that the proposed method \method also helps to narrow down the gap between hyperbolic models and Euclidean GAT.

\paragraph{Link Prediction}
The experimental results of link prediction tasks are summarized in \tab~\ref{tab:lp}. In the link prediction task, we further have the following observations: (1) Compared with Euclidean counterparts, Our proposed \method, and other hyperbolic models have achieved better performance.  It is because hyperbolic space owns a larger embedding space, where the structural dependencies could be well preserved by the link prediction loss, providing more space or boundary for nodes to be well arranged; (2) In comparison with the advanced hyperbolic models, our model also obtains remarkable gains and refresh the records.

According to the above extensive experiments, we are safe to conclude that equipping ORC with hyperbolic geometry further improves its generalization ability, obtaining high-quality representations for both tree-like and non-tree-like structured data. This confirms our primary motivation that the curvature carries rich information which is beneficial for graph representation learning in the embedding manifold. 
Specifically, incorporating the structure information featured by ORC helps the models developed in a  continuous manifold with negative curvature to  perceive the role of each node, accelerating the learning procedure, and reducing the distortion for graph embedding of less hierarchical networks.

\subsection{Effectiveness of Aggregations}
\begin{table}[t]
\centering
\caption{Performance of different aggregation methods: curvature ({\sc Curv}) and feature-enhanced curvature ({\sc CurvAtt}).}
\resizebox{0.43\textwidth}{!}{%
\begin{tabular}{@{}lcccc@{}}
\toprule
{\sc Dataset}  & \multicolumn{2}{c}{NC} & \multicolumn{2}{c}{LP} \\
         & {\sc Curv}               & {\sc CurvAtt}            & {\sc Curv}             & {\sc CurvAtt}          \\ \midrule
{\sc Disease}  & $91.8\pm1.9$           & $\textbf{92.3}\pm1.4$           & $95.7\pm0.3$         & $\textbf{95.8}\pm0.1$         \\
{\sc Airport}  & $92.4\pm1.0$          & $\textbf{92.8}\pm0.8$            & $98.0\pm0.1$         & $\textbf{98.2}\pm0.1$         \\
{\sc PubMed}   & $81.2\pm0.1$           & $\textbf{82.1}\pm0.4$           & $96.5\pm0.1$         & $\textbf{96.7}\pm0.1$         \\
{\sc Cora}     & $82.3\pm0.6$           & $\textbf{82.5}\pm0.6$           & $94.9\pm0.3$         & $\textbf{95.0}\pm0.5$         \\ \bottomrule
\vspace{-20pt}
\end{tabular}%
}
\label{tab:aggregation}
\end{table}
In Section~\ref{sec:agg_method}, we introduce two tangential aggregation strategies: the curvature-based approach (denoted as {\sc Curv}) and the feature-augmented method (denoted as {\sc CurvAtt}). The performance of these two aggregation strategies is evaluated and reported in Table~\ref{tab:aggregation}. Our results show that the feature-augmented approach outperforms the curvature-based one in most cases, which can be attributed to the complex nature of real-world systems and the incongruities between node features and topology. The feature-augmented method offers a more flexible and adaptive way to synthesize information from various sources, thus resulting in improved performance.

\begin{figure}[!tp]
\centering
\includegraphics[width=0.40\textwidth]{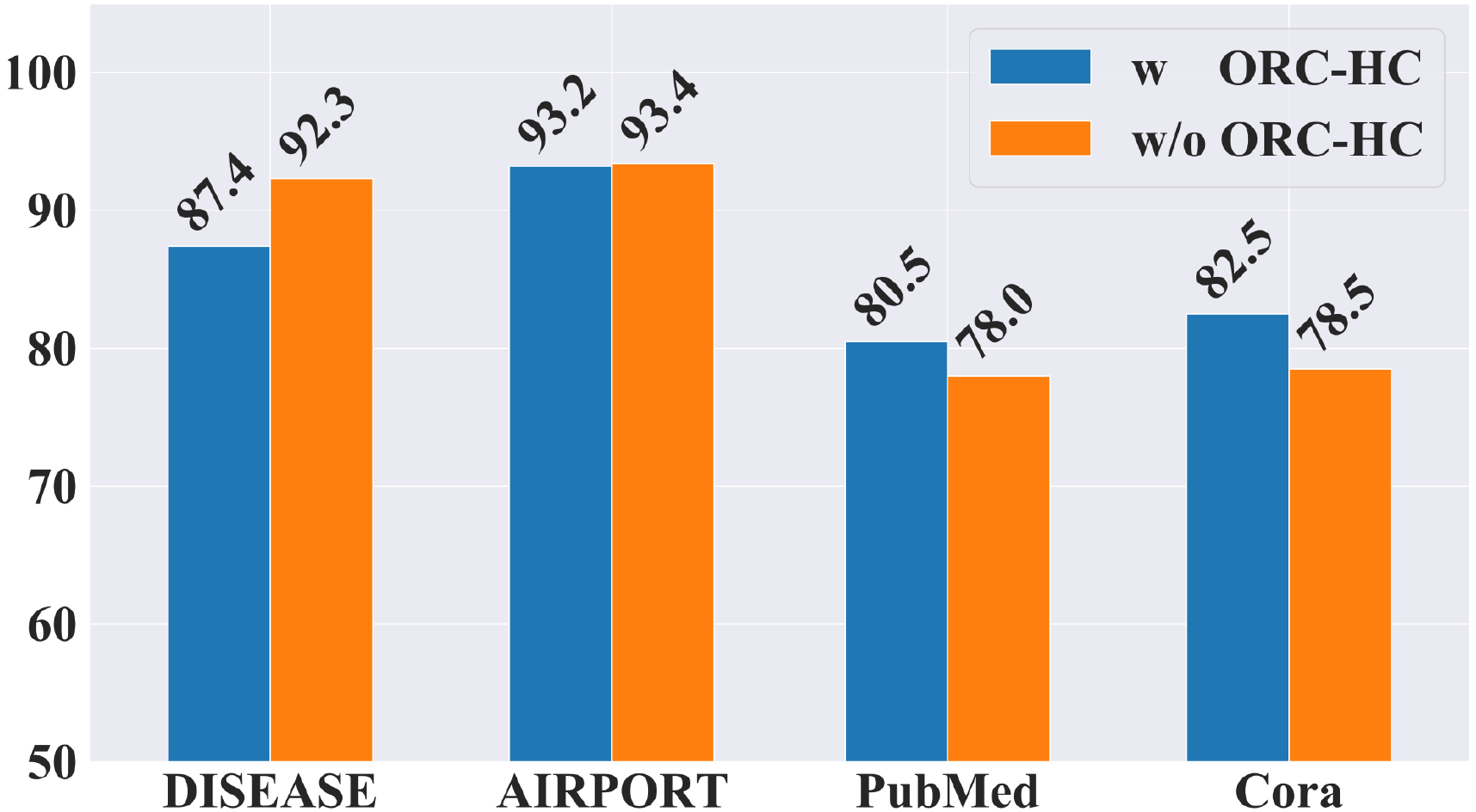}
\small
\caption{The performance of node classification by \method with \hr and without \hr.}
\label{fig:att_curv_lp_decoder}
\end{figure}

\subsection{Effectiveness of \hr} 
\label{sec:orc_hc_study}
Figure~\ref{fig:att_curv_lp_decoder}  displays the results of adding \hr or not on \method. As it observed, the performance degenerates substantially on DISEASE when applying \hr, while there are significant improvements on the three citation networks, i.e., PubMed, and Cora. 
This phenomenon can be understood as follows. Adding \hr as in node classification will force linked nodes to obtain more similar representations. For the pure tree-like dataset, i.e., DISEASE (without any triangle and circle), these node pairs belong to different levels, and adding \hr impairs the learning of asymmetric dependencies, which further affects the establishment of hierarchical awareness. When it comes to the citation networks instead, \hr helps to reduce the distortion caused by hyperbolic geometry and thus boost the learning.
\begin{figure}[!tp]
\centering
\begin{minipage}[t]{0.5\textwidth}
\centering
\includegraphics[width=6.8cm]{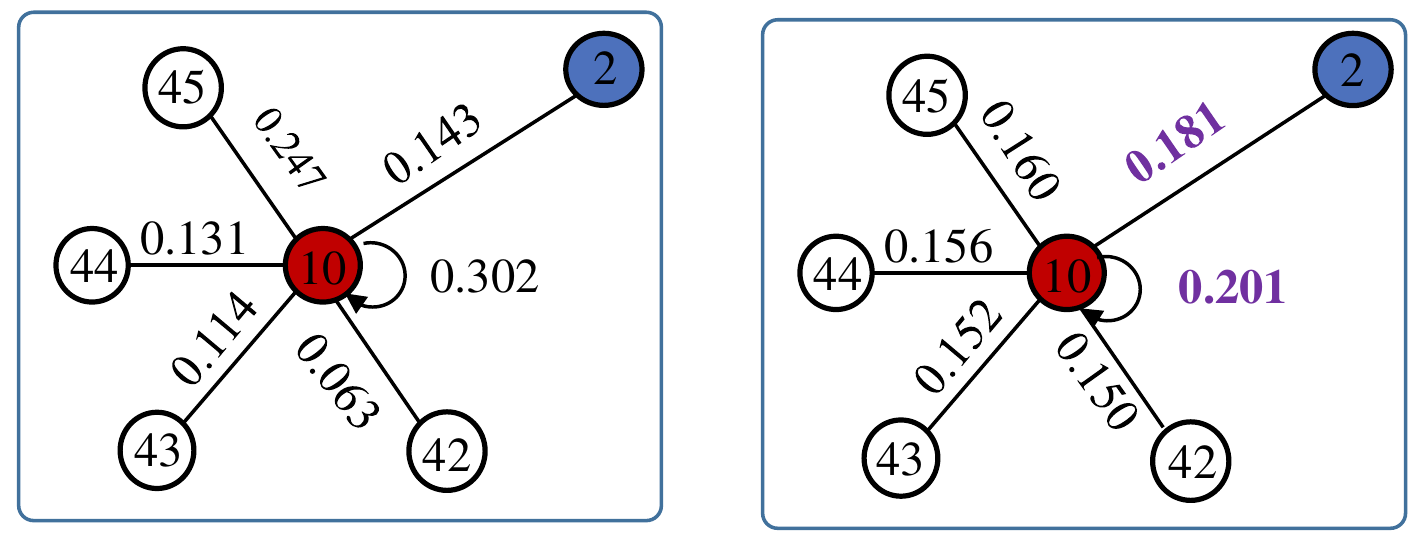}

\includegraphics[width=6.8cm]{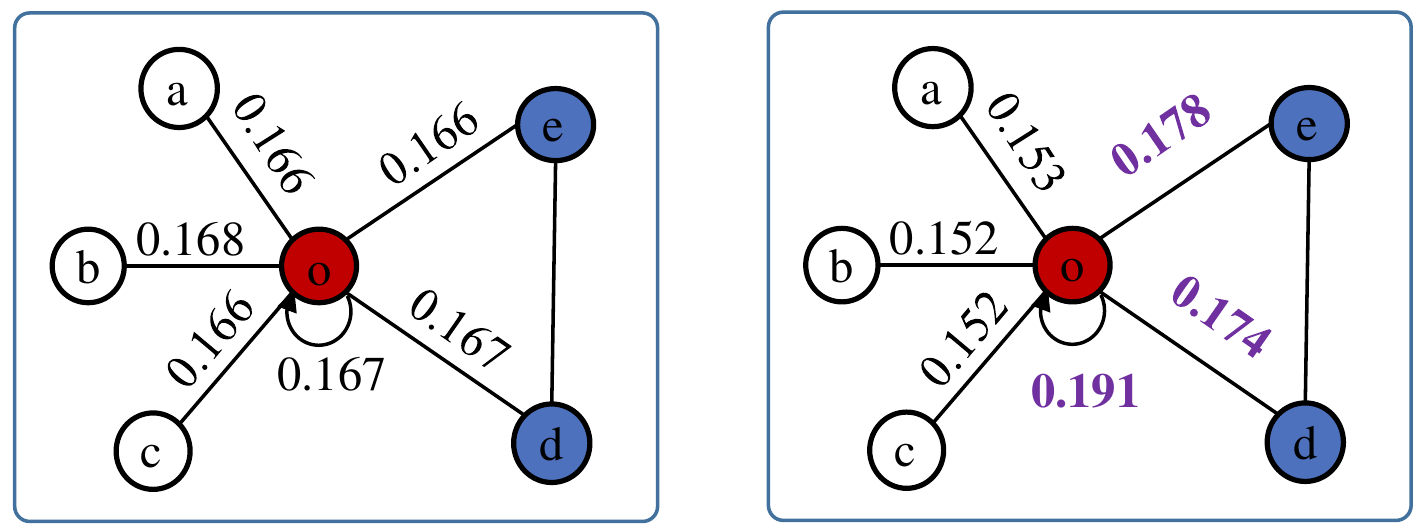}
\caption{A tree-like area (upper row) and a triangle-contained area (lower row) with edge weights by HGCN (left column) and \method (right column).}
\label{fig:case_study}

\end{minipage}
\vspace{-10pt}
\end{figure}

\subsection{Case Study of Ricci Curvature Weights}
In this section, we demonstrate the effectiveness of our method through a case study. We extract a subtree centered on a randomly sampled node (node 10, in this case), from the {\sc Disease} dataset, where node 2 is from a higher level and the remaining nodes (42, 43, 44, 45) belong to a lower level. The edges depict disease propagation paths. As shown in the upper sub-figures of Figure~\ref{fig:case_study}, we display the corresponding edge weights assigned by both \method and HGCN during the node classification task. The comparison between the two reveals that \method (upper right) effectively distinguishes node levels, as it assigns greater importance to the parent node (node 2) and equally emphasizes the child nodes from the same level, whereas HGCN (upper left) fails to make such distinctions. This substantiates the importance of ORC in facilitating hierarchical learning.

Furthermore, we illustrate a subgraph with triangles selected from the citation network, {\sc Cora}, as depicted in the lower two sub-figures of \fig~\ref{fig:case_study}. This subgraph comprises nodes ${o,e,d}$ that form a triangle. We examine the edge weights around node $o$ and present the results in the lower two sub-figures of \fig~\ref{fig:case_study}. Observing these results, it is evident that \method (as shown in the lower right) effectively identifies the local triangle structure and assigns larger weights to promote inter-node message exchange. In contrast, HGCN (as depicted in the lower left) fails to grasp the intricate topology in this area. These observations further attest to the efficacy of our proposed method in uncovering local clusters and mitigating the distortions imposed by hyperbolic geometry.


\section{Conclusion}
For modeling tree-like structures, the hyperbolic space has demonstrated its ability to capture hierarchical relationships. However, approximating a discrete tree-like graph with a hyperbolic manifold can result in inevitable distortions, as real-world tree-like graphs are inherently complex. In this work, we integrate the intrinsic graph structure into the continuous hyperbolic embedding space via the discrete Ricci curvature. 
As expected, the graph curvature facilitates the node to perceive the role and the hierarchy it belongs to, helping accelerate the hierarchical formation as well as alleviate the distortion in local clusters or cliques. 
The superiority of the proposal is demonstrated by extensive experiments. Curvature is a geometric notion with appealing and descriptive properties for both network and continuous space. Via the interaction of curvatures, we can build proper connections for a graph and the embedding space to obtain high-quality representations, which is  a promising direction to advance geometric learning. In the future, we will consider the use of Ricci flow, a more sophisticated geometric concept derived from Ricci curvature, to further enhance graph embedding and graph machine learning in tree-likeness modeling.

\section*{Acknowledgements}
We express gratitude to the anonymous reviewers and area chairs for their valuable comments and suggestions. This work was partly supported by grants from the National Key Research and Development Program of China (No.~2018AAA0100204) and the Research Grants Council of the Hong Kong Special Administrative Region, China (CUHK 14222922, RGC GRF, No. 2151185).
\newpage
\bibliographystyle{ACM-Reference-Format}
\bibliography{sample-base}
\appendix
\newpage
\section*{Appendix}

\section{Datasets}
\label{appendix:datasets}
In this section, we provide details about the datasets used in our study. The {\sc Disease} dataset contains nodes that are labeled as infected or not infected with a disease, with features indicating their susceptibility to the disease. The disease-spreading network in this dataset displays a clear hierarchical structure, which makes it ideal for testing the effectiveness of hyperbolic embedding models. We use this dataset to validate our proposal. The {\sc Airport} dataset consists of nodes representing airports, with edges indicating the existence of routes between two airports and labels reflecting the population of the respective country that the airport belongs to. In contrast, the {\sc PubMed} and {\sc Cora} datasets represent scientific papers as nodes, with edges indicating citations and labels corresponding to academic subfields.
 Additionally, the {\sc Disease} and {\sc Airport} datasets have imbalanced node labels, rendering accuracy an inadequate measure of model performance. Therefore, we use the F1 score as a more suitable measure for imbalanced datasets. Conversely, for the remaining datasets with balanced node classes, we use accuracy to evaluate the models. Furthermore, note that according to the official code of HGCN\footnote{\url{https://github.com/HazyResearch/hgcn}}, the data statistics for {\sc Disease} in node classification and link prediction differ slightly, and we list them separately in Table~\ref{tab:statistics} for clarity.

Hyperbolicity $\delta$ is a metric that quantifies how closely a graph resembles a tree structure, with lower values of $\delta$ indicating greater tree-like characteristics. 
A $\delta$ value of 0 corresponds to a tree, while higher hyperbolicity values indicate a less tree-like structure. The relationship between hyperbolicity and tree-like characteristics has been established in various studies, including~\cite{adcock2013tree,narayan2011large,jonckheere2011congestion}.
 \begin{table*}[htp!]
    \centering
    \caption{Summary of operations in the Poincar{\'e} ball model and the Lorentz model ($\kappa<0$)}
    \resizebox{0.85\textwidth}{!}{%
    \begin{tabular}{lcc}
    \toprule
    & \textbf{Poincar{\'e} Ball Model} & \textbf{Lorentz Model (Hyperboloid Model)} \\ \midrule \midrule
        \textbf{Manifold}  & $\mathcal{B}_{\kappa}^{n}=\left\{\mathbf{x} \in \mathbb{R}^{n}:\langle \mathbf{x}, \mathbf{x}\rangle_{2}<-\frac{1}{\kappa}\right\}$ 
    &  $\mathcal{L}_{\kappa}^{n}=\left\{\mathbf{x} \in \mathbb{R}^{n+1}:\langle \mathbf{x}, \mathbf{x}\rangle_{\mathcal{L}}=\frac{1}{\kappa}\right\}$          \\
    \textbf{Metric}     &  $g_{\mathbf{x}}^{\mathcal{B}_{\kappa}}=\left(\lambda_{\mathbf{x}}^{\kappa}\right)^{2} \mathbf{I}_n \text { where } \lambda_{\mathbf{x}}^{\kappa}=\frac{2}{1+\kappa\|\mathbf{x}\|_{2}^{2}}$ & $g_{\mathbf{x}}^{\mathcal{L}_{\kappa}}=\eta, \text { where } \eta \text { is } I \text { except } \eta_{0,0}=-1$    \\
\textbf{Distance}   &$d_{\mathcal{B}}^{\kappa}(\mathbf{x}, \mathbf{y})=\frac{1}{\sqrt{|\kappa|}} \cosh ^{-1}\left(1-\frac{2 \kappa\|\mathbf{x}-\mathbf{y}\|_{2}^{2}}{\left(1+\kappa\|\mathbf{x}\|_{2}^{2}\right)\left(1+\kappa\|\mathbf{y}\|_{2}^{2}\right)}\right)$
    &  $d_{\mathcal{L}}^{\kappa}(\mathbf{x}, \mathbf{y})=\frac{1}{\sqrt{|\kappa|}} \cosh ^{-1}\left(\kappa\langle \mathbf{x}, \mathbf{y}\rangle_{\mathcal{L}}\right)$  \\
    \textbf{Logarithmic Map}    &$\log _{\mathbf{x}}^{\kappa}(\mathbf{y})=\frac{2}{\sqrt{|\kappa| \lambda^{\kappa}}} \tanh ^{-1}\left(\sqrt{|\kappa|}\left\|-\mathbf{x} \oplus_{\kappa} \mathbf{y}\right\|_{2}\right) \frac{-\mathbf{x} \oplus_{\kappa} \mathbf{y}}{\left\|-\mathbf{x} \oplus_{\kappa} \mathbf{y}\right\|_{2}}$              
    &$\log _{\mathbf{x}}^{\kappa}(\mathbf{y})=\frac{\cosh ^{-1}\left(\kappa\langle \mathbf{x}, \mathbf{y}\rangle_{\mathcal{L}}\right)}{\sinh \left(\cosh ^{-1}\left(\kappa\langle \mathbf{x}, \mathbf{y}\rangle_{\mathcal{L}}\right)\right)}\left(\mathbf{y}-\kappa\langle \mathbf{x}, \mathbf{y}\rangle_{\mathcal{L}} \mathbf{x}\right)$\\ 
    \textbf{Exponential Map}   & $\exp _{\mathbf{x}}^{\kappa}(\mathbf{v})=\mathbf{x} \oplus_{\kappa}\left(\tanh \left(\sqrt{|\kappa|} \frac{\lambda_{\mathbf{x}}^{\kappa}\|\mathbf{v}\|_{2}}{2}\right) \frac{\mathbf{v}}{\sqrt{|\kappa|\|\mathbf{v}\|_{2}}}\right)$
    & $\exp _{\mathbf{x}}^{\kappa}(\mathbf{v})=\cosh \left(\sqrt{|\kappa|}\|\mathbf{v}\|_{\mathcal{L}}\right) \mathbf{x}+\mathbf{v} \frac{\sinh \left(\sqrt{|\kappa|}\|\mathbf{v}\|_{\mathcal{L}}\right)}{\sqrt{|\kappa||| \mathbf{v}||_{\mathcal{L}}}}$ \\
    \textbf{Parallel Transport} &  $P T_{\mathbf{x} \rightarrow \mathbf{y}}^{\kappa}(\mathbf{v})=\frac{\lambda_{\mathbf{x}}^{\kappa}}{\lambda_{\mathbf{y}}^{\kappa}} \operatorname{gyr}[\mathbf{y},-\mathbf{x}] v $& $ P T_{\mathbf{x} \rightarrow \mathbf{y}}^{\kappa}(\mathbf{v})=\mathbf{v}-\frac{\kappa\langle \mathbf{y}, \mathbf{v}\rangle_{\mathcal{L}}}{1+\kappa\langle \mathbf{x}, \mathbf{y}\rangle_{\mathcal{L}}}(\mathbf{x}+\mathbf{y})$ \\
    \bottomrule
    \end{tabular}
    \vspace{-10pt}
    }
    \label{tab:operation}
    \end{table*}   


 \section{Hyperbolic Geometry}
 \label{appendix:sec:hyperbolic_geometry}
The geometry of a Riemannian manifold is defined by its curvature: elliptic geometry for positive curvature, Euclidean geometry for zero curvature, and hyperbolic geometry for negative curvature. In this study, we concentrate on the latter, i.e. hyperbolic geometry. There exist several equivalent hyperbolic models that exhibit diverse characteristics, yet are mathematically isometric. Our main focus will be on two extensively researched hyperbolic models.: the Poincaré ball model \cite{nickel2017poincare} and the Lorentz model (also known as the hyperboloid model) \cite{nickel2018learning}. Let $\| . \|$ be the Euclidean norm and  $\left\langle .,. \right \rangle_\mathcal{L}$ denote the Minkowski inner product, respectively. The two models are denoted by Definition~\ref{def:poincare} and Definition~\ref{def:lorentz}. A compilation of the formulas and operations associated with these models, such as distance, mapping, and parallel transport, is presented in Table~\ref{tab:operation}. These operations include M\"obius addition~\citep{ungar2007hyperbolic} denoted by $\oplus_{\kappa}$ and the gyration operator~\citep{ungar2007hyperbolic} denoted by $\operatorname{gyr}[.,.] v$.
    \begin{definition}[Poincaré Ball Model]
    The Poincaré Ball Model with negative curvature $\kappa$ is defined as a Riemannian manifold $(\mathcal{B}_{\kappa}^{n},g_{\mathcal{B}})$, where $\mathcal{B}_{\kappa}^{n}$ is an open $n$-dimensional ball with radius $1/\sqrt{|\kappa|}$, $\mathcal{B}_{\kappa}^{n}=\left \{\mathbf{x}\in \mathbb{R}^{n}:\| \mathbf{x} \|^2<-1/\kappa \right \}$. The metric tensor of this model is expressed as $g_\mathcal{B}=\lambda^2g_E$, where the conformal factor $\lambda=\frac{2}{1+\kappa|\mathbf{x}|^2}$ and $g_E=I_n$ represents the Euclidean metric. 
    \label{def:poincare}
    \end{definition}
    
\begin{definition}[Lorentz Model]
 The Lorentz model, also known as the hyperboloid model, is defined as a Riemannian manifold $(\mathcal{L}_{\kappa}^{n}, g_{\mathcal{L}})$, where $\mathcal{L}_{K}^{n} = \left \{  \mathbf{x}\in \mathbb{R}^{n+1}: \left \langle \mathbf{x},\mathbf{x} \right \rangle_{\mathcal{L}}=\frac{1}{\kappa} \right \}$. The metric tensor of the Lorentz model is given by $g_{\mathcal{L}}=\text{diag}([-1,1,...,1])$, where $\kappa$ is the negative curvature constant.
 \label{def:lorentz}
\end{definition}

\section{More Analysis}
\subsection{Computational Complexity of ORC}
\begin{table}[!tp]
\centering
\caption{Computation cost of Ricci curvature where the unit of time is in seconds. NC: node classification; LP: link prediction.}
\resizebox{0.35\textwidth}{!}{%
\begin{tabular}{@{}lcccc@{}}
\toprule
{\sc Task} & {\sc Disease} & {\sc Airport} & {\sc PubMed} & {\sc Cora} \\ \midrule
NC (s)                      & 0.8                            & 2.7                           & 45.5                          & 2.3                         \\
LP (s)                      & 1.5                            & 2.2                           & 41.3                          & 2.1                         \\ \bottomrule
\vspace{-10pt}
\end{tabular}%
}
\label{tab:curvature_computing}
\end{table}
The computation of ORC is formulated as linear programming problems~\cite{cushing2019graph,loisel2014ricci}, and its computational complexity is $O(|E|d_{\max}^3)$, where $d_{\max}$ is the maximum degree of the graph. To illustrate the computational cost, we present the actual run-time on a machine with the environment {Intel(R) Xeon(R) Gold 6132 CPU @ 2.60GHZ} in \tab~\ref{tab:curvature_computing}. The results indicate that the time required for computing ORC is proportional to the size of the graph, and the cost for smaller graphs is correspondingly lower. Importantly, it is noteworthy that ORC only needs to be computed once prior to the training process, with the same computational complexity as HGCN~\cite{hgcn2019} during both training and inference. To handle extremely large-scale graphs, approximation methods such as Sinkhorn~\cite{cuturi2013sinkhorn} or Jaccard proxy~\cite{Jaccard2017efficient} may be utilized.

\subsection{Embedding Visualization}
\begin{figure}[ht]
	\centering
	\subfigure{
        \begin{minipage}[t]{0.36\linewidth}
        \centering
        \includegraphics[width=\linewidth]{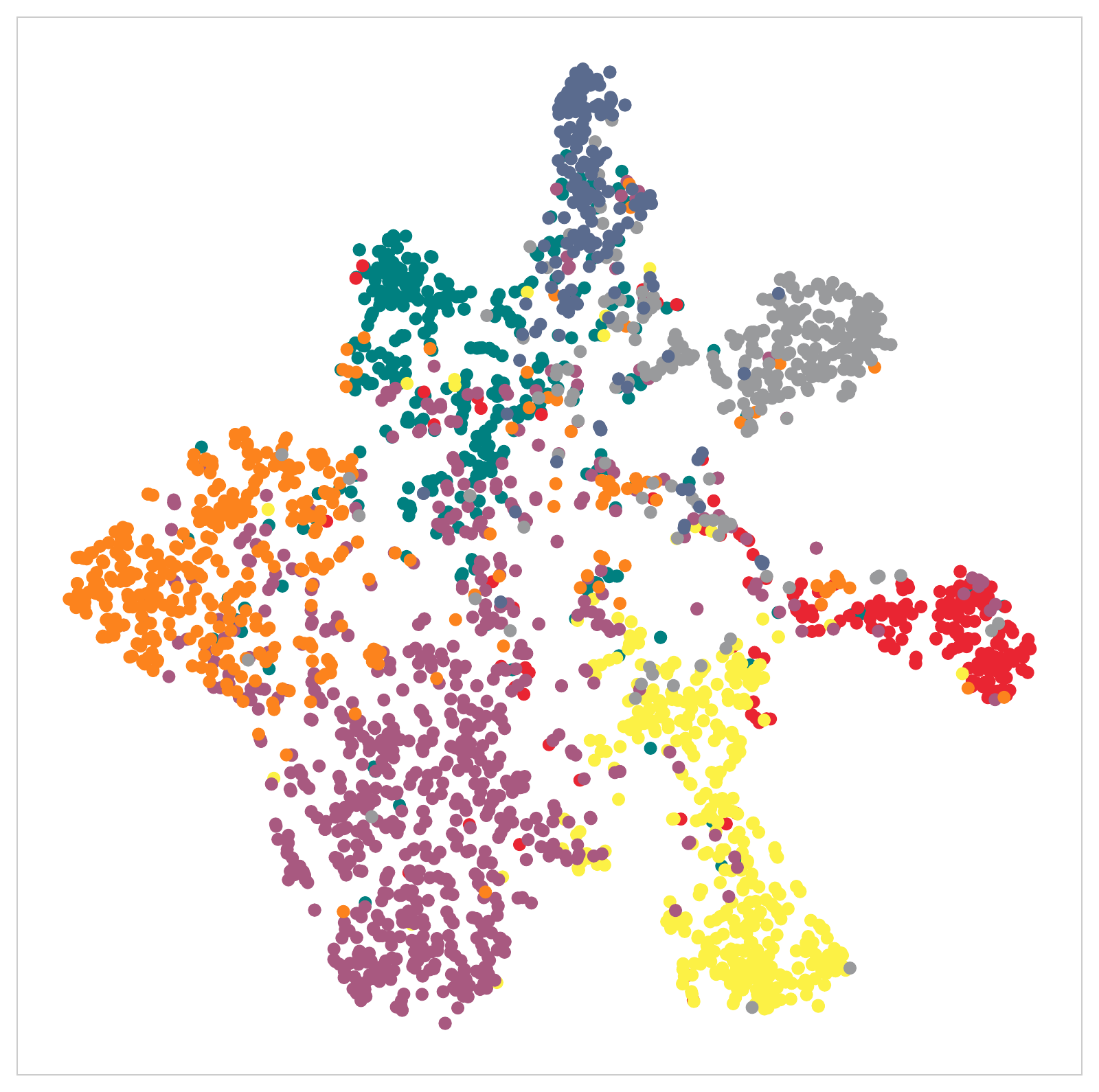}
        \end{minipage}%
        \hspace{30.5pt}
        \begin{minipage}[t]{0.36\linewidth}
        \centering
        \includegraphics[width=\linewidth]{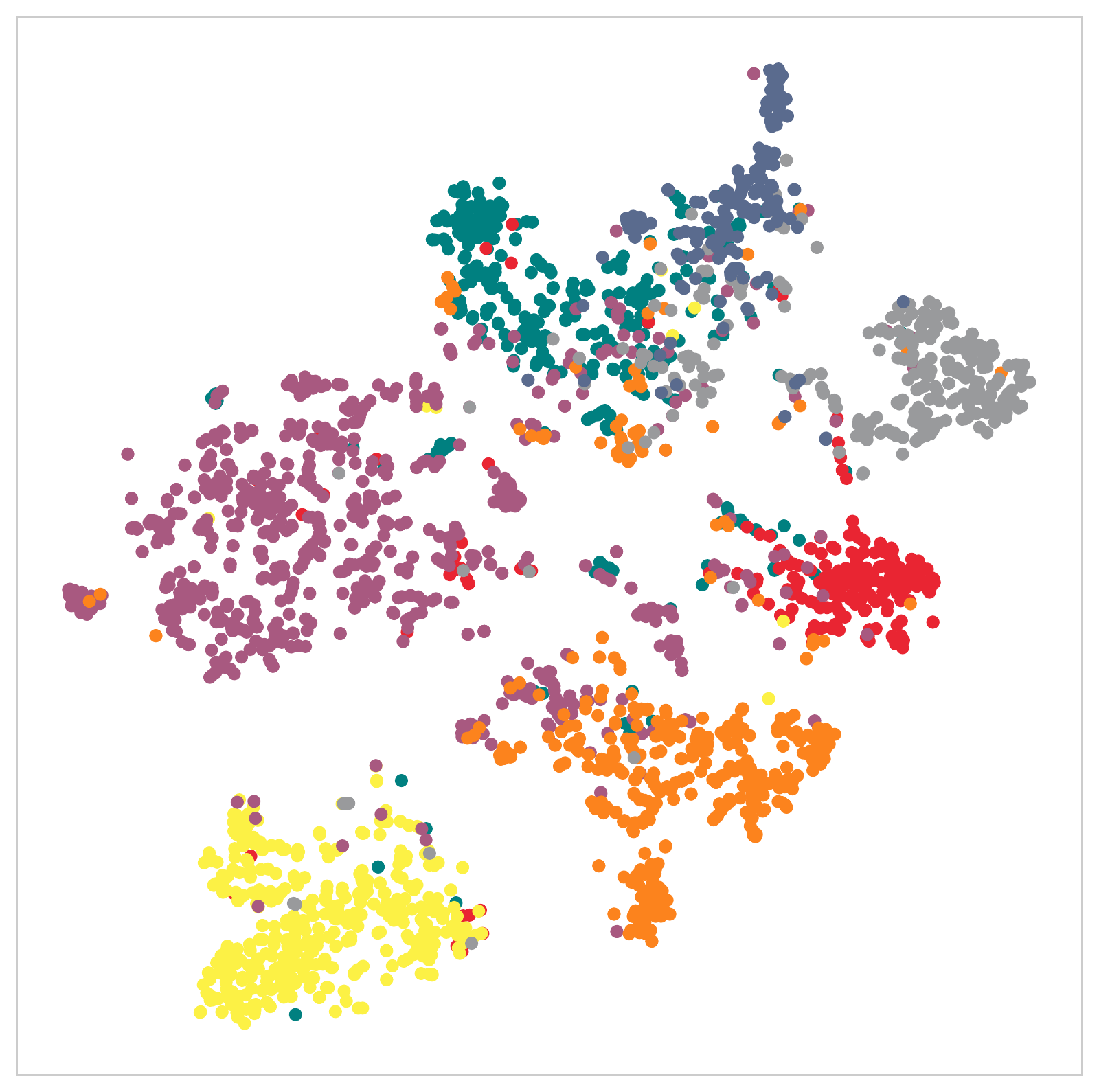}
        \end{minipage}%
        }
    \subfigure{
        \begin{minipage}[t]{0.36\linewidth}
        \centering
        \includegraphics[width=\linewidth]{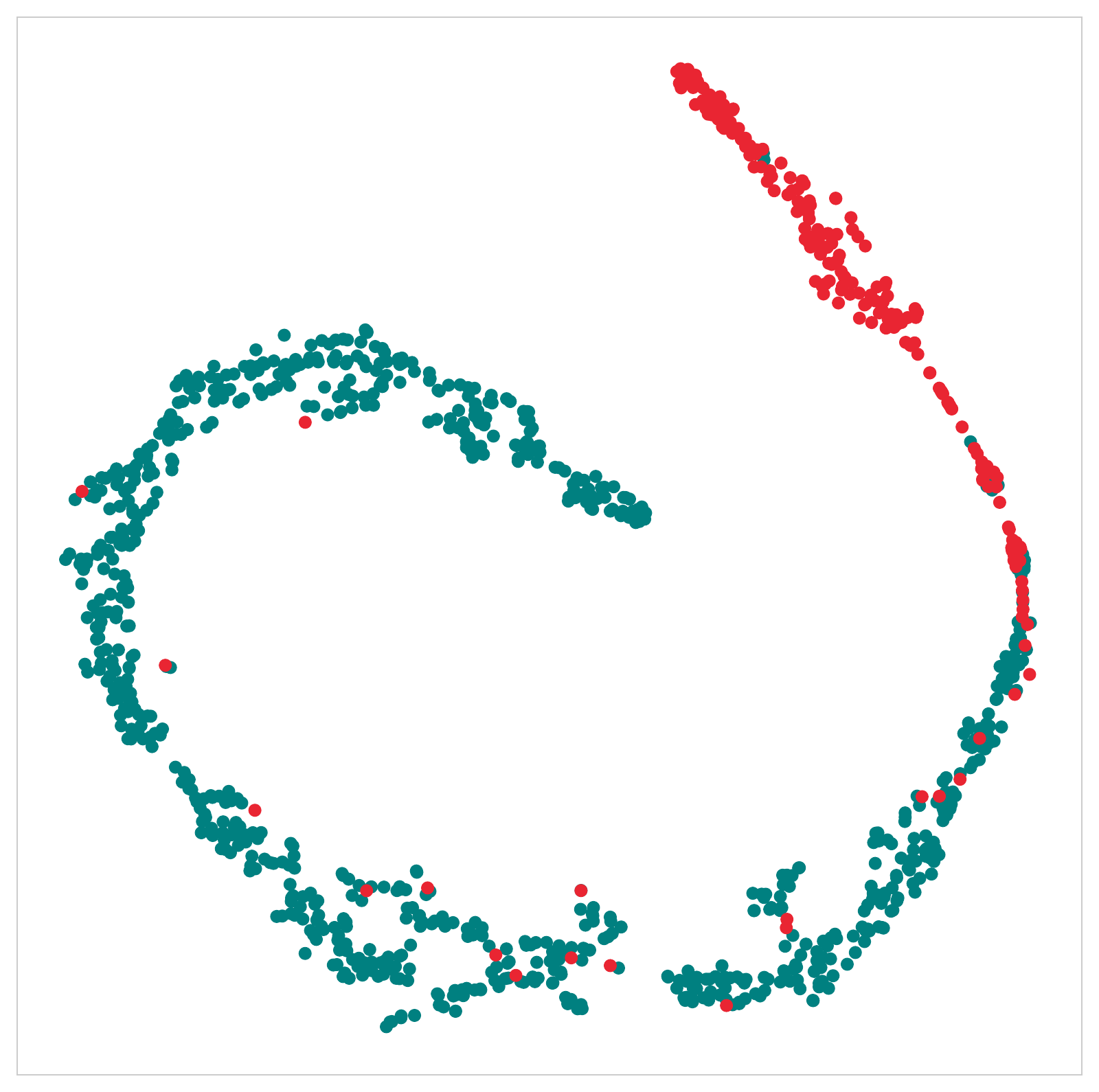}
        \end{minipage}%
        \hspace{30.5pt}
        \begin{minipage}[t]{0.36\linewidth}
        \centering
        \includegraphics[width=\linewidth]{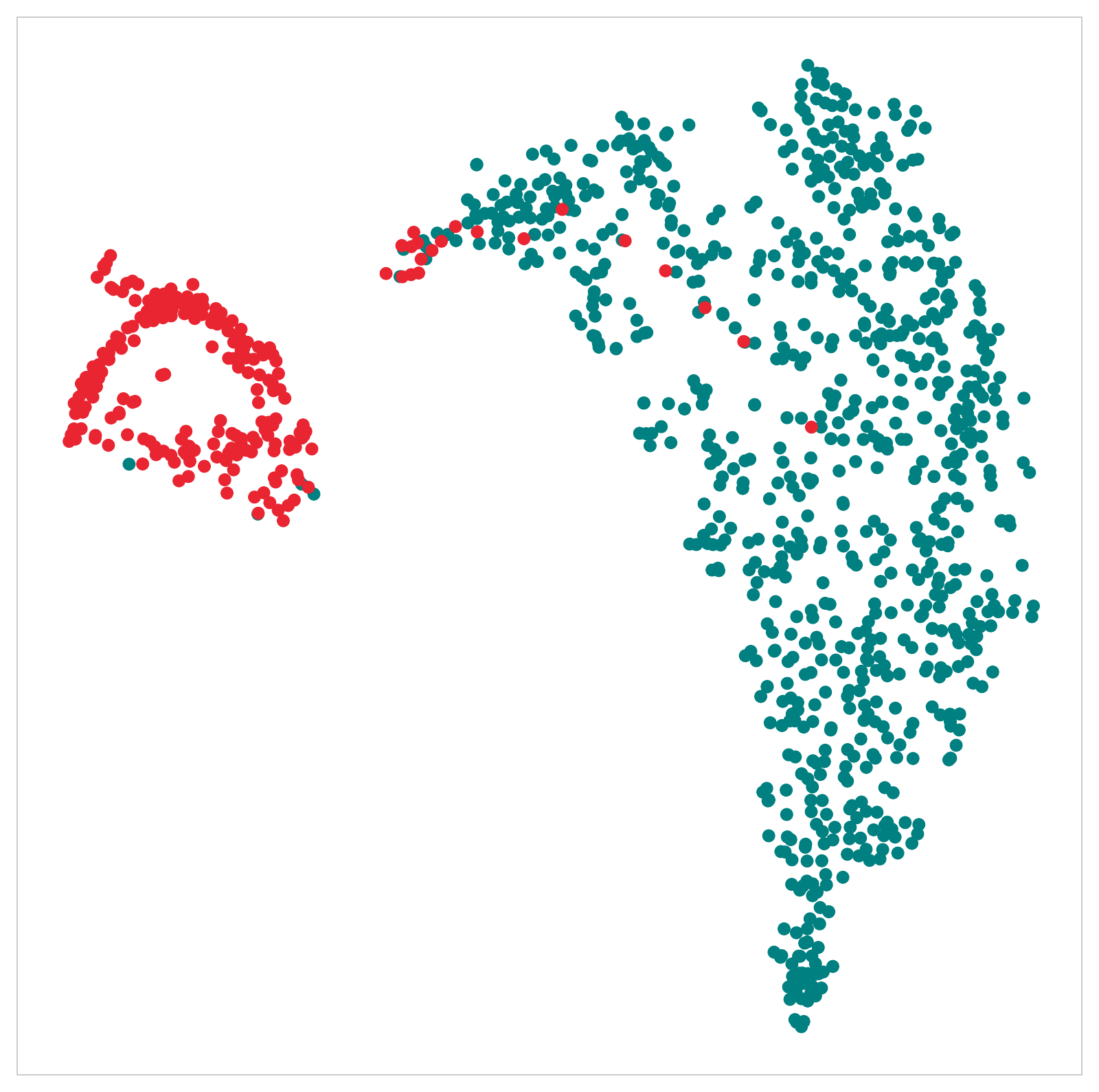}
        \end{minipage}%
    }%
    \small
	\caption{Visualization of Cora and DISEASE. \textbf{First row:} Embedding of Cora by HGCN (left) and \method (right)];~\textbf{Second row:} Embedding of DISEASE by HGCN (left) and \method (right)}
	\label{fig:visual}
	\vspace{-10pt}
    \end{figure}
The efficacy of \method and HGCN in learning representations for node classification is demonstrated through the visualization of their performance on the {\sc Disease} and {\sc Cora} datasets. To accomplish this, we employ the t-distributed Stochastic Neighbor Embedding (t-SNE) technique~\cite{maaten2008visualizing} to reduce the high-dimensional embeddings produced by the final layer of each model to a two-dimensional plane for visual examination. The results, shown in {Figure}~\ref{fig:visual}, depict nodes as individual points, where each point is assigned a color that corresponds to its class. The visualization indicates that the representations learned by \method exhibit sharper boundaries between different classes, thus showcasing the improved discriminative power of the proposed method compared to HGCN.

\subsection{Curvature-wise Performance}
The objective of the study is to perceive the local structure around nodes in tree-like graphs, encompassing local tree-like, zero-density, and densely connected structures. To demonstrate the effectiveness of the proposed method, we conducted further analysis by classifying nodes into defined local substructures. As an example, consider the following, we set the nodes to the following three types:
tree-like $\left(\kappa_i \geq-0.01\right)$; zero-like $\left(-0.01<\kappa_i \leq 0.01\right)$; positive-like $\left(\kappa_i>0.01\right)$.
The curvature of each node is defined as the sum of curvatures over its edges, that is:
$
\kappa_i = \frac{1}{\left|N_i\right|} \sum_{j \in N_i} \kappa_{ij},
$
where $\left|N_i\right|$ is the number of neighbors of node $i$, and $\kappa_{ij}$ is the curvature of the edge between nodes $i$ and $j$. This curvature measure was used to determine the local substructure of each node. In the following, we evaluated the performance of three models (GCN, HGCN, and the proposed method) on both the {\sc Cora} and {\sc Airport} datasets. Specifically, we calculated the accuracy/F1-score on the test set for nodes in each local substructure.

\begin{table}[!tp]
\centering
\caption{Curvature-wise performance on {\sc cora} dataset}
\begin{tabular}{cccc} 
\toprule
Method & Negative & Zero & Positive \\
\midrule
GT-Prop (\%) & $\mathbf{0.536}$ & $\mathbf{0.138}$ & $\mathbf{0.326}$ \\ \midrule\midrule 
GCN (\%) & 0.512 & $\mathbf{0.085}$ & 0.216 \\
HGCN (\%) & 0.523 & 0.077 & 0.213 \\
\method (\%) & $\mathbf{0.535}$ & 0.082 & $\mathbf{0.217}$ \\
\bottomrule
\vspace{-10pt}
\end{tabular}
\label{tab:cora_results}
\end{table}

\begin{table}[ht]
\centering
\caption{Curvature-wise performance on {\sc Airport} dataset}
\begin{tabular}{cccc} 
\toprule
Method & Negative & Zero & Positive \\
\midrule
GT-Prop (\%) & $\mathbf{0.418}$ & $\mathbf{0.439}$ & $\mathbf{0.143}$ \\ \midrule \midrule
GCN (\%) & 0.277 & $\mathbf{0.424}$ & 0.118 \\
HGCN (\%) & 0.382 & 0.405 & 0.117 \\
\method (\%) & $\mathbf{0.395}$ & 0.410 & $\mathbf{0.123}$ \\
\bottomrule
\vspace{-20pt}
\end{tabular}
\label{tab:airport_results}
\end{table}

The results are shown in Table~\ref{tab:cora_results} and Table~\ref{tab:airport_results}. The "GT-Prop" row in the Tables show the proportion of nodes belonging to each local substructure, with the proportions summing to 1. The values in the GCN, HGCN, and proposed model rows represent the proportion of nodes in each substructure that were correctly predicted by the respective models. In other words, the table presents the accuracy of each model in representing the local environment for different types of nodes.
Overall, the study found that the models achieved comparable accuracy to the best achievable, but performance varied across different substructures. Specifically, Euclidean GCN outperformed HGCN on zero-density and densely connected areas, but underperformed on tree-like nodes. In contrast, HGCN showed improved accuracy for tree-like nodes, but lower accuracy on other substructures. The proposed model achieved a balance of performance across all substructures, with high accuracy for both tree-like and non-tree-like nodes.

\subsection{Connections of the Discrete and Continuous Curvatures} The discrete curvature $\kappa$ is computed in advance, which can be regarded as an edge weight. The continuous curvature $c$ of the predefined hyperbolic space is learnable and differentiable. For simplicity, we denote the learnable parameters of our model \method\ as $\theta$. Let us consider node classification as an example. If the ground-truth label of node $\mathbf{x}$ is $\mathbf{y}$ and the predicted label is $\bar{\mathbf{y}}$, then we have $\bar{\mathbf{y}} = \textrm{\method}(\mathbf{x}, \kappa, \theta, c)$ and the loss is given by $L({\mathbf{y},\bar{\mathbf{y}}})$. We can take the derivative of $c$ with respect to the loss, i.e., $\frac{\partial L({\mathbf{y},\bar{\mathbf{y}}})}{\partial c}=\frac{\partial(\mathbf{y},\textrm{\method}(\mathbf{x}, \kappa, \theta, c))}{\partial c}$. It is easy to see that the update of $c$ is constrained by $\kappa$. In other words, we learn a good embedding space equipped with curvature $c$ that matches the graph structure through discrete Ricci curvature $\kappa$.

\end{document}